\documentclass[11pt, en]{elegantpaper}

\everymath{\displaystyle}

\newcommand{\ROM}[1]{\uppercase\expandafter{\romannumeral #1}}

\newcommand{\secref}[1]{Section~\ref{#1}}

\newcommand{\tabref}[1]{Table~\ref{#1}}

\newcommand{\defref}[1]{Definition~\ref{#1}}
\newcommand{\aspref}[1]{Assumption~\ref{#1}}
\newcommand{\thmref}[1]{Theorem~\ref{#1}}
\newcommand{\colref}[1]{Corollary~\ref{#1}}
\newcommand{\lemref}[1]{Lemma~\ref{#1}}
\newcommand{\proref}[1]{Proposition~\ref{#1}}

\newenvironment{proofof}[1]{
    \begin{proof}[\normalfont\textbf{Proof of #1}]
        }{
    \end{proof}
}

\DeclareCiteCommand{\mycite}
{\usebibmacro{prenote}}
{\ifciteindex
    {\indexnames{labelname}}
    {}%
    \printnames[first-last,firstinits=true,
        terseinits=true]{labelname}%
    \addcomma\space
    \printfield{prefixnumber}%
    \printfield{labelnumber}}
{\multicitedelim}
{\usebibmacro{postnote}}
\newcommand{\dd}{\mathrm{d}}
\newcommand{\ee}{\mathrm{e}}

\newcommand{\bbE}{\mathbb{E}}
\newcommand{\bbN}{\mathbb{N}}
\newcommand{\bbR}{\mathbb{R}}

\newcommand{\calB}{\mathcal{B}}
\newcommand{\calE}{\mathcal{E}}
\newcommand{\calH}{\mathcal{H}}
\newcommand{\calN}{\mathcal{N}}
\newcommand{\calX}{\mathcal{X}}
\newcommand{\calY}{\mathcal{Y}}

\newcommand{\bsz}{\boldsymbol{z}}

\newcommand{\LE}{\left}
\newcommand{\RI}{\right}

\newcommand{\eqspace}{
    \hphantom{{}={}\!}
}  

\newcommand{\norm}[1]{\LE\| #1 \RI\|}
\newcommand{\inner}[1]{\LE\langle #1 \RI\rangle}

\newcommand{\family}[1]{\LE\{ #1 \RI\}}
\newcommand{\ind}[1]{\mathbf{1}_{\family{#1}}}

\newcommand{\E}[1]{\operatorname{\bbE} \LE[ #1 \RI]}

\newcommand{\sE}[2]{\operatorname{\bbE}_{#2} \LE[ #1 \RI]}

\newcommand{\Tr}[1]{\operatorname{Tr} \LE( #1 \RI)}

\let\oldforall\forall
\renewcommand{\forall}{\mathrel{\oldforall}}

\let\oldexists\exists
\renewcommand{\exists}{\mathrel{\oldexists}}

\newcommand{\frho}{f_{\rho}}

\newcommand{\flam}{f_{\lambda}}
\newcommand{\hflam}{\widehat{f}_{\bsz, \lambda}}

\newcommand{\urho}{u_{\rho}}

\newcommand{\rhoS}{\rho^{\textnormal{S}}}
\newcommand{\rhoT}{\rho^{\textnormal{T}}}

\newcommand{\rhoSX}{\rhoS_{\calX}}
\newcommand{\rhoTX}{\rhoT_{\calX}}

\newcommand{\rhoSx}{\rhoSX(x)}
\newcommand{\rhoTx}{\rhoTX(x)}

\newcommand{\Kx}{K_{x}}
\newcommand{\Kxa}{K_{x}^{\ast}}

\newcommand{\LK}{L_{K}}
\newcommand{\LKlam}{L_{K, \lambda}}
\newcommand{\hLK}{\widehat{L}_{K}}
\newcommand{\hLKlam}{\widehat{L}_{K, \lambda}}

\newcommand{\bsy}{\boldsymbol{y}}

\newcommand{\philam}{g_{\lambda}}

\newcommand{\HS}{\mathrm{HS}}

\newcommand{\wD}{w^{\dagger}}

\newcommand{\LKD}{\LK^{\dagger}}
\newcommand{\hLKD}{\hLK^{\mspace{1mu} \dagger}}
\newcommand{\LKlamD}{\LKlam^{\dagger}}
\newcommand{\hLKlamD}{\hLKlam^{\dagger}}

\newcommand{\hflamD}{\hflam^{\mspace{5mu} \dagger}}

\newcommand{\ranH}[1]{[\calH]^{#1}}

\newcommand{\hSKa}{\widehat{S}_{K}^{\ast}}
\newcommand{\hSKDa}{\widehat{S}_{K}^{\mspace{2mu} \dagger})^{\ast}}

\makeatletter
\newcommand{\biggg}{\bBigg@\thr@@}
\newcommand{\Biggg}{\bBigg@{3.5}}
\makeatother

\newcommand{\neweqref}[2]{\textup{(\hyperref[#1]{\textbf{#2}})}}

\begin{document}

\title{Spectral Algorithms in Misspecified Regression: \\ Convergence under Covariate Shift\( ^{\dagger} \)\footnotetext{\dag~The work is partially supported by National Natural Science Foundation of China [Project No. 12271473, No. U21A20426]. The corresponding author is Zheng-Chu Guo. Email addresses: \href{mailto:12335032@zju.edu.cn}{12335032@zju.edu.cn} (R. R. Liu), \href{mailto:guozc@zju.edu.cn}{guozc@zju.edu.cn} (Z. C. Guo).}}
\author[]{Ren-Rui Liu}
\author[]{Zheng-Chu Guo}
\affil[]{School of Mathematical Sciences, Zhejiang University, Hangzhou 310058, China}
\date{}

\maketitle

\begin{abstract}


    This paper investigates the convergence properties of spectral algorithms—a class of regularization methods originating from inverse problems—under covariate shift. In this setting, the marginal distributions of inputs differ between source and target domains, while the conditional distribution of outputs given inputs remains unchanged. To address this distributional mismatch, we incorporate importance weights, defined as the ratio of target to source densities, into the learning framework. This leads to a weighted spectral algorithm within a nonparametric regression setting in a reproducing kernel Hilbert space (RKHS). More importantly, in contrast to prior work that largely focuses on the well-specified setting, we provide a comprehensive theoretical analysis of the more challenging misspecified case, in which the target function does not belong to the RKHS. Under the assumption of uniformly bounded density ratios, we establish minimax-optimal convergence rates when the target function lies within the RKHS. For scenarios involving unbounded importance weights, we introduce a novel truncation technique that attains near-optimal convergence rates under mild regularity conditions, and we further extend these results to the misspecified regime. By addressing the intertwined challenges of covariate shift and model misspecification, this work extends classical kernel learning theory to more practical scenarios, providing a systematic framework for understanding their interaction.
\end{abstract}


\section{Introduction}

Supervised learning, a cornerstone of modern machine learning, aims to develop predictive models from labeled training data drawn from a source distribution. Classical statistical learning theory establishes that, under the idealized assumption of identical source and target distributions, basic empirical risk minimization can learn a function that generalizes well to unseen target instances \cite{Vapnik1998StatisticalLT}. However, practical applications frequently violate such distributional stationarity. Temporal variations, sampling biases, or environmental changes often create discrepancies between source and target distributions, a challenge collectively termed \emph{distribution shift} or \emph{dataset shift} \cite{Candela2008DatasetSM}. This phenomenon represents a critical obstacle to robust machine learning deployment.

This work focuses specifically on \emph{covariate shift}, a prevalent form of distribution shift characterized by differing marginal distributions while maintaining identical conditional distributions. Such shifts arise in many practical scenarios. For example, medical datasets collected across different hospitals may exhibit varying demographic compositions (covariate distributions), while diagnostic criteria for individual patients (conditional distributions) remain consistent \cite{Finlayson2021ClinicianDS}. Various techniques have been developed to mitigate the effects of covariate shifts. Among these, importance weighting \cite{Shimodaira2000ImprovingPI}, which adjusts the source data based on the density ratio, appears particularly effective, and there has been a number of work investigating its theoretical properties \cite{Gizewski2022RegularizationUD, Ma2023OptimallyTC, Feng2023TowardsUA, Gogolashvili2023WhenIW, Fan2025SpectralAU}.

We formalize our problem within a regression framework utilizing the square loss. Given a training sample \( \bsz = \family{ (x_{i}, y_{i}) }_{i=1}^{n} \) drawn from a source distribution \( \rhoS(x, y) \), where \( x_{i} \in \calX \), \( y_{i} \in \calY \), the input space \( \calX \) is a separable and compact metric space, and the output space \( \calY \) is a subset of \( \bbR \), our goal is to find a predictor \( f \) that minimizes the expected risk on the target distribution \( \rhoT(x, y) \):
\[
    \calE(f) = \sE{( \, y - f(x) \, )^{2}}{(x, y) \sim \rhoT}.
\]

Under covariate shift, a setting where marginal distributions \( \rhoSx \) and \( \rhoTx \) differ but conditional distribution \( \rho(y \mid x) \) remains identical, the source and target distributions factorize as:
\[
    \rhoS(x, y) = \rho(y \mid x) \rhoSx, \quad \rhoT(x, y) = \rho(y \mid x) \rhoTx
\]
with \( \rhoSX \ne \rhoTX \). In this setting, the optimal predictor over all measurable functions is the regression function
\[
    f_{\rho}(x) = \int y \, \dd \rho(y \mid x).
\]
Nevertheless, learning over the entire space of measurable functions is infeasible in practice, making the specification of a suitable hypothesis space fundamental. Here, we work within reproducing kernel Hilbert spaces (RKHS). Specifically, let \( K \colon \calX \times \calX \to \bbR \) be a Mercer kernel, which is a continuous, symmetric, and positive semi-definite function. This kernel induces an RKHS \( \calH \) with the reproducing property
\[
    f(x) = \inner{ f, K(\cdot, x) }_{\calH},
    \quad \forall f \in \calH.
\]
Additionally, we assume uniform boundedness \( \sup_{x \in \calX} K(x, x) \leq \kappa^{2} \), where \( \kappa \ge 1 \). The regression problem is categorized based on the relationship between \( \frho \) and \( \calH \): it is \emph{well-specified} when \( \frho \in \calH \), and \emph{misspecified} otherwise. In the latter case, misspecification typically implies reduced regularity of \( \frho \), which introduces difficulties in the learning problem.

To approximate \( \frho \) within the RKHS \( \calH \), we minimize the expected risk \( \sE{ ( \, y - f(x) \, )^{2} }{(x, y) \sim \rhoT} \) over \( f \in \calH \). As established in prior work (e.g., Proposition~2 of \cite{Vito2005RiskBR}), this minimization is equivalent to solving the operator equation:
\[
    \LK \, f =\LK \, \frho,
    \quad f \in \calH,
\]
where \( L_K \) is the integral operator defined as
\[
    \LK \colon L^{2}(\calX, \rhoTX) \to L^{2}(\calX, \rhoTX),
    \quad f \mapsto \int_{\calX} f(x) K(\cdot, x) \, \dd \rhoTx.
\]
Note that \( \calH \) continuously embeds into \( L^{2}(\calX, \rhoTX) \), hence \( \LK \) acts on \( f \in \calH \). Given only a finite sample \( \bsz = \family{ (x_{i}, y_{i}) }_{i=1}^{n} \) drawn from \( \rhoS \), the empirical version of this equation is expressed as
\begin{equation}
    \label{eq:empirical}
    \frac{1}{n} \sum_{i=1}^{n} w(x_{i}) f(x_{i}) \, K(\cdot, x_{i})
    = \frac{1}{n} \sum_{i=1}^{n} w(x_{i}) y_{i} \, K(\cdot, x_{i}),
\end{equation}
where we use the Radon-Nikodym derivative (commonly referred to as the density ratio) \( w \) to weight the sample:
\[
    w(x) = \frac{\dd \rhoTX}{\dd \rhoSX}(x).
\]
This strategy is known as \emph{importance weighting} \cite{Shimodaira2000ImprovingPI}, which aligns the source and target distributions in expectation. For notational simplicity, we define the empirical integral operator:
\[
    \hLK \colon \calH \to \calH,
    \quad f \mapsto \frac{1}{n} \sum_{i=1}^{n} w(x_{i}) f(x_{i}) \, K(\cdot, x_{i}),
\]
and the adjoint of the sampling operator:
\[
    \hSKa \colon \bbR^{n} \to \calH,
    \quad \bsy \to \frac{1}{n} \sum_{i=1}^{n} w(x_{i}) y_{i} \, K(\cdot, x_{i}),
    \quad \bsy = (y_{1}, \dots, y_{n})^{\top}.
\]
The empirical equation \eqref{eq:empirical} then simplifies to:
\[
    \hLK \, f = \hSKa \, \bsy.
\]

Since \( \hLK \) is generally non-invertible, regularization is required to solve \eqref{eq:empirical}. To address this, we employ \emph{spectral algorithms}—a class of regularization techniques designed to produce a stable inverse operator. Originally developed for ill-posed linear inverse problems (see, e.g., \cite{Engl2015RegularizationIP}), these algorithms bridge learning theory and inverse problems, rendering them highly effective for regression tasks \cite{Vito2005LearningFE}. Regularization is achieved through filter functions that amplify significant eigencomponents while suppressing less influential ones:
\begin{definition}[Filter functions]
    \label{def:filter}
    A family of functions \( \philam \colon [0, \kappa^2] \to [0, \infty) \), parameterized by \( \lambda > 0 \), constitutes filter functions if:
    \begin{itemize}
        \item[(1)] There exists \( E \geq 0 \) such that for all \( \theta \in [0, 1] \):
              \begin{equation}
                  \label{eq:philam}
                  \sup_{t \in [0, \kappa^2]} t^\theta \philam(t) \leq E \lambda^{\theta - 1}.
              \end{equation}
        \item[(2)] There exist \( \tau \geq 1 \) and \( F \geq 0 \) such that for all \( \theta \in [0, \tau] \):
              \begin{equation}
                  \label{eq:psilam}
                  \sup_{t \in [0, \kappa^2]} t^\theta |1 - t \philam(t)| \leq F \lambda^\theta.
              \end{equation}
    \end{itemize}
\end{definition}
Intuitively, condition \eqref{eq:philam} ensures that the regularized inverse defined by \( \philam(t) \) remains bounded, thereby guaranteeing numerical stability. Condition \eqref{eq:psilam} controls the approximation error by requiring the residual term \( |1 - t \philam(t)| \) to vanish at a prescribed rate as \( \lambda \to 0 \). The parameter \( \tau \), known as the qualification of the regularization method, quantifies the maximum degree of source smoothness that the spectral algorithm can effectively exploit. Specifically, it characterizes the class of target functions for which optimal convergence rates are attainable (see \aspref{asp:source}). Through the filter function framework, spectral algorithms encompass a broad family of regularization methods. Common examples include:
\begin{itemize}
    \item Kernel ridge regression:
          \( \philam^{\text{krr}}(t) = (t + \lambda)^{-1} \), with \( \tau = 1 \) and \( E = F = 1 \); the regularization parameter is \( \lambda \).

          \smallskip

    \item Early-stopped gradient descent:
          \( \philam^{\text{gf}}(t) = t^{-1}(1 - \ee^{-t/\lambda}) \), with arbitrary \( \tau \ge 1 \) and \( E = 1, F = (\tau / \ee)^\tau \); the stopping time is \( 1/\lambda \).

          \smallskip

    \item Spectral cutoff:
          \( \philam^{\text{cut}}(t) = t^{-1} \ind{t \geq \lambda} \), with arbitrary \( \tau \ge 1 \) and \( E = F = 1 \); the cutoff threshold is \( \lambda \).
\end{itemize}
Given a filter function \( \philam \), the weighted spectral algorithm then takes the form:
\begin{equation}
    \label{eq:spectral alg}
    \hflam = \philam(\hLK) \, \hSKa \, \bsy.
\end{equation}
Recently, Gizewski et al. \cite{Gizewski2022RegularizationUD} studied spectral algorithms \eqref{eq:spectral alg} under covariate shift, assuming that the density ratio is uniformly bounded. Ma et al. \cite{Ma2023OptimallyTC} and Feng et al. \cite{Feng2023TowardsUA} investigated kernel ridge regression---a special case of spectral algorithms \eqref{eq:spectral alg}---under the condition that the density ratio is either bounded or unbounded but has finite second moment. Gogolashvili et al. \cite{Gogolashvili2023WhenIW} also studied kernel ridge regression, but employed a more general moment condition on the density ratio, as detailed in \aspref{asp:ratio}. Fan et al. \cite{Fan2025SpectralAU} adopted the moment condition from \cite{Gogolashvili2023WhenIW} and analyzed spectral algorithms \eqref{eq:spectral alg} within covariate shift. All these works are confined to well-specified settings, where \( \frho \in \calH \).
\begin{remark}
    Without covariate shift and importance weighting, the empirical integral operator \( \hLK \) has an operator norm uniformly bounded by \( \kappa^{2} \) under any sampling, ensuring that it is compact, self-adjoint, and positive. Thus, filter functions apply to \( \hLK \) as intended. Under covariate shift, however, the density ratio \( w \) may be unbounded, potentially causing the eigenvalues of \( \hLK \) to exceed the filter's domain \( [0, \kappa^{2}] \), thereby causing the estimator \( \hflam \) ill-defined. Nevertheless, with appropriate moment conditions on \( w \) (\aspref{asp:ratio}), concentration inequalities demonstrate that as \( n \to \infty \), \( \norm{ \hLK } \) can be bounded arbitrarily close to \( \norm{ \LK } \) with high probability (see \lemref{lem:hLK - LK}). Therefore, we may proceed with our analysis under sampling scenarios where \( \philam(\hLK) \) remains well-defined, assuming without loss of generality that \( \norm{ \hLK } \le \kappa^{2} \). We can still establish probabilistic bounds while maintaining mathematical rigor.
\end{remark}

This paper analyzes the approximation ability of \( \hflam \) to \( \frho \) under covariate shift, with a particular emphasis on the case of misspecification of \( \frho \). Our work advances the machine learning theory by making two principal contributions:
\begin{enumerate}
    \item This paper presents a unified theoretical framework for analyzing the convergence of spectral algorithms under covariate shift. This framework establishes explicit connections between the degree of model misspecification (quantified via a source condition parameter) and the severity of the distribution shift (characterized by moment conditions on the density ratio). Specifically, when the density ratio is uniformly bounded, our analysis achieves minimax optimal convergence rates (\colref{col:p=infty}). Moreover, when the underlying kernel possesses favorable embedding properties, we demonstrate that near-optimal convergence rates remain attainable even for scenarios involving unbounded density ratios (\colref{col:alpha_0=1/beta}).

    \item We introduce a truncation scheme specifically designed to handle unbounded density ratios. This scheme enables spectral algorithms to achieve near-optimal convergence rates when the regression problem is well-specified (\colref{col:truncated}). Notably, fast convergence rates for misspecified scenarios are also provided.
\end{enumerate}

The remainder of this paper is organized as follows: \secref{sec:theorem} first states necessary assumptions, then presents our main theorems and corollaries. \secref{sec:discussion} provides a literature review and a comparative analysis with existing works; \secref{sec:proof} proves the main theorems, with auxiliary lemmas deferred to \hyperref[sec:appendix]{Appendix}.

\section{Main Results}
\label{sec:theorem}

In this section, we establish the convergence rates of the weighted spectral algorithm estimator \( \hflam \) to the regression function \( \frho \). We begin by presenting key assumptions for our convergence analysis. The first assumption, adopted from \cite{Gogolashvili2023WhenIW}, characterizes the severity of covariate shift through moment conditions on the density ratio \( w \).
\begin{assumption}[Moment of density ratio]
    \label{asp:ratio}
    Let \( w = \dd \rhoTX / \dd \rhoSX \) denote the density ratio. There exist constants \( p \in [1, \infty] \), \( L > 0 \), and \( \sigma > 0 \) such that the following moment condition holds:
    \[
        \LE( \int_{\calX} w^{p(m-1)}(x) \, \dd \rhoTx \RI)^{1/p}
        \le \frac{1}{2} m! L^{m-2} \sigma^{2},
        \quad \forall m \ge 2.
    \]
    When \( p = \infty \), the left-hand side is defined as \( \norm{ w^{m-1} }_{\infty} \).
\end{assumption}

In \aspref{asp:ratio}, we use \( \norm{ \cdot }_{\infty} \) as shorthand for \( \norm{ \cdot }_{L^{\infty}(\calX, \rhoTX)} \). This assumption quantifies distributional discrepancy by controlling the growth of density ratio moments. When \( w(x) \) is uniformly bounded on \( \calX \), the assumption holds with \( p = \infty \) and \( L = \sigma^{2} = \norm{ w }_{\infty} \). For unbounded density ratios, validity may still hold for finite \( p \in [1, \infty) \), with smaller \( p \) accommodating heavier tails. The extremal case \( p = 1 \) requires finite moments of all orders for the density ratio.

Intuitively, \aspref{asp:ratio} ensures the source distribution does not exhibit excessive deviation from the target distribution, and parameter \( p \) quantifies permissible tail behavior. For instance, if
\[
    2 \rhoTX \LE( \family{ x: w(x) \ge t } \RI) \le \sigma^{2} \exp \LE( -\frac{t^{p}}{L} \RI),
\]
then Assumption \ref{asp:ratio} is satisfied (see Proposition~12 in \cite{Gogolashvili2023WhenIW}).

Before stating the remaining assumptions, we recall some necessary background on kernel theory. Mercer's theorem (see, e.g., Theorem~4.10 in \cite{Cucker2007LearningTA}) states that a Mercer kernel \( K \) admits the following decomposition:
\begin{equation}
    \label{eq:decomp K}
    K(x, x') = \sum_{j \in N} t_{j} \, e_{j}(x) e_{j}(x'),
    \quad N \in \bbN \cup \family{ \infty },
\end{equation}
where \( \family{ t_{j} }_{j \in N} \) and \( \family{ e_{j} }_{j \in N} \) are the eigenvalues and eigenfunctions of \( \LK \). The sequence of eigenvalues \( \family{ t_{j} }_{j \in N} \) is non-negative and non-increasing, and the eigenfunctions \( \family{ e_{j} }_{j \in N} \) forms an orthonormal system in \( L^{2}(\calX, \rhoTX) \). Furthermore, \( \family{ t_{j}^{1/2} \, e_{j} }_{j \in N} \) constitutes an orthonormal basis for the reproducing kernel Hilbert space \( \calH \), and the embedding \( \calH \hookrightarrow L^{2}(\calX, \rhoTX) \) is continuous. For \( \gamma \in (0, 1) \), the scaled system \( \family{ t_{j}^{\gamma/2} \, e_{j} }_{j \in N} \) spans an intermediate space between \( \calH \) and \( L^{2}(\calX, \rhoTX) \), known as an interpolation space \cite{Steinwart2012MercerTG}:
\begin{definition}[Interpolation spaces]
    Let \( \family{ t_{j} }_{j \in N} \) and \( \family{ e_{j} }_{j \in N} \) be as in \eqref{eq:decomp K}. For \( \gamma \in [0, 1] \), the interpolation space \( \ranH{\gamma} \) is defined as:
    \[
        \ranH{\gamma}
        = \operatorname{span} \family{ t_{j}^{\gamma/2} \, e_{j} }_{j \in N}
        = \family{ \sum_{j \in N} f_{j} \, (t_{j}^{\gamma/2} \, e_{j}): f_{j} \in \bbR, \sum_{j \in N} f_{j}^{2} < \infty }.
    \]
    The inner product on \( \ranH{\gamma} \) is given by
    \[
        \inner{ \sum_{j \in N} f_{j} \, (t_{j}^{\gamma/2} \, e_{j}), \sum_{j \in N} g_{j} \, (t_{j}^{\gamma/2} \, e_{j}) }_{\ranH{\gamma}}
        = \sum_{j \in N} f_{j} \, g_{j}.
    \]
\end{definition}

This framework satisfies \( \ranH{1} = \calH \) and \( \ranH{0} \subseteq L^{2}(\calX, \rhoTX) \), with continuous embeddings \( \ranH{\gamma_{2}} \hookrightarrow \ranH{\gamma_{1}} \) for any \( 0 \le \gamma_{1} < \gamma_{2} \le 1 \). These spaces unify our convergence analysis: \( \gamma = 1 \) corresponds to the well-specified case (i.e., \( \frho \in \calH \)), while smaller values of \( \gamma \) accommodate misspecification (i.e., \( \frho \in L^{2}(\calX, \rhoTX) \setminus \calH \)).

Using the decomposition \eqref{eq:decomp K}, the integral operator \( \LK \) admits the following eigen decomposition:
\[
    \LK \colon L^{2}(\calX, \rhoTX) \to L^{2}(\calX, \rhoTX),
    \quad f \mapsto \sum_{j \in N} t_{j} \, \inner{ f, e_{j} }_{\rhoTX} \, e_{j}.
\]
Here, \( \inner{ \cdot, \cdot }_{\rhoTX} \) and \( \norm{ \cdot }_{\rhoTX} \) denote the inner product and norm in \( L^{2}(\calX, \rhoTX) \), respectively. This leads to an equivalent characterization of \( \ranH{\gamma} \) via the operator \( \LK^{\gamma/2} \):
\begin{definition}[Interpolation spaces (equivalent definition)]
    \label{def:interpolation}
    Let \( \LK \) be the integral operator associated with \( K \). For \( \gamma \in [0, 1] \),
    \[
        \ranH{\gamma}
        = \operatorname{ran} \LK^{\gamma/2}
        = \family{ \LK^{\gamma/2} \, f: f \in L^{2}(\calX, \rhoTX) }.
    \]
    The inner product is given by
    \[
        \inner{ \LK^{\gamma/2} \, f, \LK^{\gamma/2} \, g }_{\ranH{\gamma}}
        = \inner{ f, g }_{\rhoTX}.
    \]
\end{definition}

Now, we can characterize the regularity of \( \frho \):
\begin{assumption}[Source condition]
    \label{asp:source}
    Let \( \tau \) be the qualification parameter in \defref{def:filter}. There exists \( r \in (0, \tau] \) such that
    \[
        \frho \in \ranH{2r} \cap L^{\infty}(\calX, \rhoTX),
    \]
    with \( \norm{ \frho }_{\infty} \le G \) and \( \urho \in L^{2}(\calX, \rhoTX) \) satisfying
    \[
        \frho = \LK^{r} \, \urho.
    \]
\end{assumption}

Regarding \aspref{asp:source}, the case \( r \ge 1/2 \) (well-specified) implies \( \frho \in \calH \), while \( r < 1/2 \) (misspecified) requires specialized treatment. The representation \( \frho = \LK^{r} \, \urho \) follows standard practice in the literature (see, e.g., \cite{Cucker2007LearningTA, Caponnetto2007OptimalRR}). The boundedness condition \( \norm{ \frho }_{\infty} \le G \) is commonly employed in misspecification analyses \cite{Guo2018GradientDR, Lin2020OptimalRS}, which additionally guarantees \( |y| \le G \) holds \( \rhoT \)-a.e. While recent work \cite{Zhang2024OptimalityMS} relaxes this via \( L^{p} \)-embedding techniques, their approach does not extend to covariate shift due to potential discrepancies between \( L^{p}(\calX, \rhoTX) \) and \( L^{p}(\calX, \rhoSX) \).

Our next assumption concerns the decay rate of the eigenvalues \( \family{ t_{j} }_{j \in N} \) of \( \LK \), which fundamentally determines the capacity of the induced RKHS \( \calH \). Rapid eigenvalue decay induces a small RKHS \( \calH \), typically enabling faster learning when \( \frho \in \calH \). Conversely, slow decay corresponds to a larger \( \calH \), which may hinder the learning process but increases the chance that \( \frho \) resides within \( \calH \).
\begin{assumption}[Eigenvalue decay rate]
    \label{asp:eigenvalue}
    The eigenvalues \( \family{ t_{j} }_{j \in N} \) of \( \LK \) exhibit a polynomial decay rate of order \( \beta > 1 \). Specifically, there exist positive constants \( c \) and \( C \) such that
    \[
        c \, j^{-\beta} \le t_{j} \le C \, j^{-\beta},
        \quad \forall j \in N.
    \]
\end{assumption}
In \aspref{asp:eigenvalue}, the upper bound quantifies the capacity of \( \calH \) and determines the convergence rates in our main results, while the lower bound ensures that these rates are optimal in the minimax sense. The requirement \( \beta > 1 \) arises from the trace-class property of \( \LK \):
\[
    \sum_{j \in N} t_{j} = \Tr{\LK} \le \kappa^{2} < \infty.
\]
\aspref{asp:eigenvalue} equivalently translates to bounds on the effective dimension \cite{Caponnetto2007OptimalRR}:
\[
    \calN(\lambda) = \Tr{ (\LK + \lambda)^{-1} \LK }.
\]
Specifically, assuming \( t_{j} \asymp j^{-\beta} \) as in \aspref{asp:eigenvalue}, where \( \asymp \) denotes equivalence up to multiplicative constants, then by \lemref{lem:effective dimension}, we obtain
\begin{equation}
    \label{eq:constant effect dim}
    c_{\calN} \lambda^{-1/\beta}
    \le \calN(\lambda)
    \le C_{\calN} \lambda^{-1/\beta}.
\end{equation}

Finally, we examine embedding properties of the interpolation spaces \( \ranH{\gamma} \). The kernel boundedness \( \sup_{x \in \calX} K(x, x) \le \kappa^{2} \) implies that all functions in \( \calH \) satisfy
\[
    \sup_{x \in \calX} |f(x)|
    = \sup_{x \in \calX} \inner{ f, K(\cdot, x) }_{\calH}
    \le \kappa \, \norm{f}_{\calH},
\]
guaranteeing the continuous embedding \( \calH = \ranH{1} \hookrightarrow L^{\infty}(\calX, \rhoTX) \). However, as \( \gamma \) decreases from 1 to 0, \( \ranH{\gamma} \) expands toward \( L^{2}(\calX, \rhoTX) \), and this embedding property weakens \cite{Fischer2020SobolevNL}. Our final assumption identifies the critical transition point:
\begin{assumption}[Embedding index]
    \label{asp:embedding}
    Let \( \beta \) be eigenvalue decay rate defined in \aspref{asp:eigenvalue}. The embedding index of \( \calH \) is \( \alpha_{0} \in [1/\beta, 1) \), defined as
    \[
        \alpha_{0} = \inf_{\alpha \in [1/\beta, 1]} \family{ \alpha: \norm{ \ranH{\alpha} \hookrightarrow L^{\infty}(\calX, \rhoTX) } < \infty }.
    \]
\end{assumption}
Existing research demonstrates that the embedding index concept typically enables more accurate convergence rate analysis when \( \frho \notin \calH \) \cite{Li2022OptimalRR, Zhang2024OptimalityMS}. While \( \alpha_{0} \le 1 \) is immediate (note that we assume \( \alpha_{0} < 1 \)), it can be proved that \( \alpha_{0} \ge 1/\beta \) (see Lemma~10 in \cite{Fischer2020SobolevNL}). Examples where \( \alpha_{0} = 1/\beta \) include kernels with uniformly bounded eigenfunctions, Sobolev kernels on bounded domains with smooth boundaries, and shift-invariant periodic kernels under uniform distributions \cite{Zhang2024OptimalityMS}.

We now present our main results. The following theorem characterizes convergence rates under our general settings:
\begin{theorem}
    \label{thm:importance weighting}
    Under \aspref{asp:ratio} with \( p \in [1, \infty] \), \aspref{asp:source} with \( r \in (0, \tau] \), \aspref{asp:eigenvalue} with \( \beta > 1 \), and \aspref{asp:embedding} with \( \alpha_{0} \in [1/\beta, 1) \), let \( \lambda = n^{-s} \), where
    \[
        s =
        \begin{cases}
            \LE( 2r + \frac{1}{\beta} + \frac{\alpha_{0} + \epsilon - 1/\beta}{p} \RI)^{-1},                    & 2r > \alpha_{0}; \smallskip \\
            \LE( \alpha_{0} + \epsilon + \frac{1}{\beta} + \frac{\alpha_{0} + \epsilon - 1/\beta}{p} \RI)^{-1}, & 2r \le \alpha_{0}.
        \end{cases}
    \]
    When \( 2r > \alpha_{0} \), we take \( \epsilon \in (0, 2r - \alpha_{0}) \); otherwise any \( \epsilon > 0 \) is allowed. Then, for any \( \delta \in (0, 1) \) and
    \begin{equation}
        \label{eq:n}
        \begin{aligned}
            n \ge \max \biggg\{ & \LE( 16 L M_{\alpha_{0} + \epsilon/2}^{2} \log \frac{6}{\delta} \RI)^{\frac{1}{1 - s (\alpha_{0} + \epsilon/2)}},                                                                                                                       \\
                                & \LE( 16 \sigma M_{\alpha_{0} + \epsilon/2}^{1 + \frac{1}{p}} C_{\calN} \log \frac{6}{\delta} \RI)^{\frac{2}{1 - s \LE( \alpha_{0} + \frac{\epsilon}{2} + \frac{1}{\beta} + \frac{\alpha_{0} + \epsilon/2 - 1/\beta}{p} \RI)}} \biggg\},
        \end{aligned}
    \end{equation}
    with \( M_{\alpha_{0} + \epsilon/2} = \norm{ \ranH{\alpha_{0} + \epsilon/2} \hookrightarrow L^{\infty}(\calX, \rhoTX) } \) denoting the embedding norm, the following convergence bound holds with probability exceeding \( 1 - \delta \):
    \[
        \norm{ \hflam - \frho }_{\ranH{\gamma}}
        = O \LE( n^{-s \LE( r - \frac{\gamma}{2} \RI)} \log \frac{6}{\delta} \RI),
        \quad 0 \le \gamma \le \min \family{ 2r, 1 }.
    \]
\end{theorem}

The convergence rate in \thmref{thm:importance weighting} is jointly determined by the degree of covariate shift \( p \), the kernel and data distribution properties \( \alpha_{0}, \beta \), and the regularity \( r \) of the regression function. Although we omit the constant independent of \( n \) or \( \delta \), this constant can be obtained by carefully examining the error bounds derived in our proof (see \secref{sec:proof 1}).

Notably, when the density ratio \( w \) is uniformly bounded (indicating mild covariate shift), we obtain fast convergence rates:
\begin{corollary}
    \label{col:p=infty}
    Suppose that \aspref{asp:ratio} holds with \( p = \infty \), implying that the density ratio is uniformly bounded. Under \aspref{asp:source} with \( r \in (0, \tau] \), \aspref{asp:eigenvalue} with \( \beta > 1 \), \aspref{asp:embedding} with \( \alpha_{0} \in [1/\beta, 1) \),
    when \( n \) is sufficiently large satisfying \eqref{eq:n}, setting \( \gamma = 0 \) yields the following simplified convergence rate in \thmref{thm:importance weighting}:
    \begin{itemize}
        \item For \( 2r > \alpha_{0} \),
              \[
                  \norm{ \hflam - \frho }_{\rhoTX} = O \LE( n^{-\frac{r}{2r + 1/\beta}} \log \frac{6}{\delta} \RI);
              \]

        \item For \( 2r \le \alpha_{0} \),
              \[
                  \norm{ \hflam - \frho }_{\rhoTX} = O \LE( n^{-\frac{r}{\alpha_{0} + \epsilon + 1/\beta}} \log \frac{6}{\delta} \RI).
              \]
    \end{itemize}
\end{corollary}

Zhang et al.~\cite{Zhang2024OptimalityMS} established that for \( r > 0 \), the minimax lower bound in \( L^{2} \)-norm without covariate shift is \( O \LE( n^{-\frac{r}{2r + 1/\beta}} \RI) \). Thus, \colref{col:p=infty} achieves minimax optimality when \( 2r > \alpha_{0} \). Moreover, even with unbounded \( w \), fast convergence rates are attainable when the RKHS possesses favorable embedding properties:
\begin{corollary}
    \label{col:alpha_0=1/beta}
    Suppose that \aspref{asp:eigenvalue} holds with \( \beta > 1 \) and \aspref{asp:embedding} holds with \( \alpha_{0} = 1/\beta \). Under \aspref{asp:ratio} with \( p \in [1, \infty] \) and \aspref{asp:source} with \( r \in (0, \tau] \), setting \( \gamma = 0 \) yields the following simplified convergence rate in \thmref{thm:importance weighting}:
    \begin{itemize}
        \item For \( 2r > 1/\beta \),
              \[
                  \norm{ \hflam - \frho }_{\rhoTX} = O \LE( n^{-\frac{r}{2r + 1/\beta + \epsilon}} \log \frac{6}{\delta} \RI);
              \]

        \item For \( 2r \le 1/\beta \),
              \[
                  \norm{ \hflam - \frho }_{\rhoTX} = O \LE( n^{-\frac{r}{2/\beta + \epsilon}} \log \frac{6}{\delta} \RI),
              \]
    \end{itemize}
    when \( n \) is sufficiently large satisfying \eqref{eq:n}.
\end{corollary}

\colref{col:alpha_0=1/beta} demonstrates that when the embedding index achieves its optimal value \( \alpha_{0} = 1/\beta \), the spectral algorithm attains near-optimal rates for \( 2r > 1/\beta \), regardless of covariate shift severity (\( \forall p \in [1, \infty] \)).

As shown in \thmref{thm:importance weighting}, when \aspref{asp:ratio} holds with \( p \in [1, \infty) \) (i.e., \( w \) is unbounded) and \( \alpha_{0} > 1/\beta \), standard importance weighting strategy typically yields suboptimal rates. To address this, we employ truncated density ratios to enhance convergence \cite{Ma2023OptimallyTC, Feng2023TowardsUA, Gogolashvili2023WhenIW}. Specifically, for \( D > 0 \), define the truncated density ratio \( \wD(x) = \min \family{ w(x), D } \), which gives rise to the truncated empirical integral operator:
\[
    \hLKD \colon \calH \to \calH,
    \quad f \mapsto \frac{1}{n} \sum_{i=1}^{n} \wD(x_{i}) f(x_{i}) \, K(\cdot, x_{i}).
\]
The resulting estimator is then constructed as:
\[
    \hflamD = \philam(\hLKD)
    \, (\hSKDa \, \bsy,
\]
where the operator \( (\hSKDa \colon \bbR^{n} \to \calH \) is defined by \( (\hSKDa \, \bsy = \frac{1}{n} \sum_{i=1}^{n} \wD(x_{i}) y_{i} \, K(\cdot, x_{i}) \).
\begin{theorem}
    \label{thm:truncated weighting}
    Under \aspref{asp:ratio} with \( p \in [1, \infty) \), \aspref{asp:source} with \( r \in (0, \tau] \), and \aspref{asp:eigenvalue} with \( \beta > 1 \), consider \( m \ge 2 \) and define the truncated density ratio \( \wD(x) = \min \family{ w(x), D } \) with
    \[
        D = n^{\nu},
        \quad \nu = \frac{1}{p(m - 1) + 1}.
    \]
    Set \( \lambda = n^{-s} \), where
    \[
        s =
        \begin{cases}
            \frac{1 - \nu}{2r + 1/\beta},           & 2r > 1;
            \smallskip                                          \\
            \frac{1 - \nu}{1 + \epsilon + 1/\beta}, & 2r \le 1,
        \end{cases}
    \]
    for an arbitrarily small constant \( \epsilon > 0 \). Then, we obtain the following result: for any \( \delta \in (0, 1) \) and
    \begin{equation}
        \label{eq:(T) n}
        \begin{aligned}
            n \ge \max \biggg\{ & \LE( 2 \kappa C_{\calN}^{1/2} \LE( \frac{1}{2} m! L^{m-2} \sigma^{2} \RI)^{p/2} \RI)^{\frac{2}{p (m-1) \cdot \nu - \LE( 1 + \frac{1}{\beta} \RI) s}}, \\
                                & \LE( 32 \kappa^{2} \log \frac{6}{\delta} \RI)^{\frac{1}{1 - \nu - s}},
            \LE( 16 \sqrt{2} \kappa C_{\calN}^{1/2} \log \frac{6}{\delta} \RI)^{\frac{2}{1 - \nu - \LE( 1 + \frac{1}{\beta} \RI) s}} \biggg\},
        \end{aligned}
    \end{equation}
    the following convergence bound holds with probability at least \( 1 - \delta \):
    \[
        \norm{ \hflamD - \frho }_{\ranH{\gamma}}
        = O \LE( n^{-s \LE( r - \frac{\gamma}{2} \RI)} \log \frac{6}{\delta} \RI),
        \quad 0 \le \gamma \le \min \family{ 2r, 1 }.
    \]
\end{theorem}

In \thmref{thm:truncated weighting}, we again omit the constant independent of \( n \) or \( \delta \), which can be derived by examining \secref{sec:proof 2}. The following corollary shows that as \( m \) increases, the parameter \( \nu \) converges to 0, allowing the convergence rate to approach arbitrarily close to the minimax optimal rate \( O \LE( n^{-\frac{r - \gamma/2}{2r + 1/\beta}} \RI) \) when \( 2r > 1 \):
\begin{corollary}
    \label{col:truncated}
    Suppose \aspref{asp:ratio} holds with \( p \in [1, \infty) \) and \aspref{asp:eigenvalue} holds with \( \beta > 1 \). Let \( \epsilon > 0 \) be a fixed small constant.
    \begin{itemize}
        \item When \aspref{asp:source} holds with \( 2r > 1 \), select \( m \) sufficiently large such that
              \[
                  \nu = \frac{1}{p(m - 1) + 1} \le \frac{2r + 1/\beta}{r} \epsilon.
              \]
              Then, with the truncation level \( D = n^{\nu} \) and sufficiently large \( n \) satisfying \eqref{eq:(T) n}, evaluating at \( \gamma = 0 \) and \( \gamma = 1 \) yields the simplified convergence rates in \thmref{thm:truncated weighting}:
              \[
                  \LE\{
                  \begin{aligned}
                      \norm{ \hflamD - \frho }_{\rhoTX} & = O \LE( n^{-\frac{r}{2r + 1/\beta} + \epsilon} \log \frac{6}{\delta} \RI),       \\
                      \norm{ \hflamD - \frho }_{\calH}  & = O \LE( n^{-\frac{r - 1/2}{2r + 1/\beta} + \epsilon} \log \frac{6}{\delta} \RI).
                  \end{aligned}
                  \RI.
              \]

        \item When \aspref{asp:source} holds with \( 2r \le 1 \), choose \( m \) sufficiently large to satisfy
              \[
                  \nu = \frac{1}{p(m - 1) + 1} \le \frac{\epsilon^{2}}{r} + \LE( \frac{1 + 1/\beta}{r} - \frac{1}{1 + 1/\beta} \RI) \epsilon,
              \]
              then, assuming \eqref{eq:(T) n} holds, evaluating at \( \gamma = 0 \) yields
              \[
                  \norm{ \hflamD - \frho }_{\rhoTX} = O \LE( n^{-\frac{r}{1 + 1/\beta} + \epsilon} \log \frac{6}{\delta} \RI).
              \]
    \end{itemize}
\end{corollary}

\colref{col:truncated} establishes that in well-specified regression problems, truncation methods successfully achieve near-optimal convergence rates while handling unbounded density ratios.

\section{Related Work and Discussion}
\label{sec:discussion}

Kernel methods offer a powerful nonparametric framework for function approximation in reproducing kernel Hilbert spaces (RKHS). Foundational work by Caponnetto and de Vito \cite{Caponnetto2007OptimalRR} established minimax optimal convergence rates for kernel ridge regression (KRR) in well-specified settings, where the true regression function belongs to the RKHS. These rates depend critically on two parameters: the source condition \( r \), which characterizes the smoothness of the target function; and the eigenvalue decay rate \( \beta \) of the integral operator, which characterizes the capacity of the RKHS. Subsequent research has extended this framework to various settings, including gradient descent \cite{Raskutti2014EarlySN, Cao2024StochasticGD}, robust regression \cite{Guo2018GradientDR, Hu2022EarlySI}, and random feature methods \cite{Rudi2017GeneralizationPL, Li2019TowardsUA}. Through spectral filtering, spectral algorithms generalize KRR to encompass a broader class of regularization families \cite{Gerfo2008SpectralAS}. In well-specified regimes (\( 1/2 \le r \le \tau \)), where \( \tau \) denotes the qualification parameter, these algorithms achieve minimax optimality \cite{Guo2017LearningTD, Blanchard2018OptimalRR, Fan2024SpectralAF}. Subsequent advances address misspecification (\( 0 < r < 1/2 \)), demonstrating that spectral algorithms retain optimality under the condition \( 2r > 1 - 1/\beta \) \cite{Lin2020OptimalRS}. Recently, the embedding index \( \alpha_0 \in [1/\beta, 1] \), introduced by Fischer and Steinwart \cite{Fischer2020SobolevNL}, refines the analysis of RKHS capacity. Building on this, a broader optimality range \( 2r > \alpha_{0} - 1/\beta \) is obtained \cite{Zhang2024OptimalityMS}. Despite these advances, all the aforementioned works remain confined to identical source and target distributions.

Prior analyses of misspecified kernel methods typically yield convergence rates governed by a threshold \( \mathbf{T} \leq 1 \):
\[
    \norm{ \hflam - \frho }_{\rhoTX}
    \le
    \begin{cases}
        O \LE( n^{-\frac{r}{2r + 1/\beta}} \RI),                    & 2r > \mathbf{T};   \\
        O \LE( n^{-\frac{r}{\mathbf{T} + \epsilon + 1/\beta}} \RI), & 2r \le \mathbf{T}.
    \end{cases}
\]
Here, the rate \(O(n^{-\frac{r}{2r + 1/\beta}})\) for \(2r > \mathbf{T}\) is minimax optimal; thus, a smaller \(\mathbf{T}\) enlarges the minimax optimal range. Before the introduction of the embedding index (\aspref{asp:embedding}), the best known threshold was \( \mathbf{T} = 1 - 1/\beta \) \cite{Lin2020OptimalRS}. Recently, Zhang et al. \cite{Zhang2024OptimalityMS} improved this to \( \mathbf{T} = \alpha_{0} - 1/\beta \) using the embedding index. In contrast, our \thmref{thm:importance weighting} achieves \( \mathbf{T} = \alpha_{0} \), which appears weaker by \( 1/\beta \). This gap arises from their reliance on concentration inequalities that require uniform boundedness of empirical operators (e.g., Lemma~32 in \cite{Zhang2024OptimalityMS})---a condition that fails under covariate shift where the density ratio \( w \) may be unbounded. We conjecture that with bounded \( w \), our threshold could similarly reach \( \mathbf{T} = \alpha_{0} - 1/\beta \). As for \thmref{thm:truncated weighting}, the threshold is \( \mathbf{T} = 1 \). Although we incorporate the embedding index \( \alpha_{0} \) in our proof (see \secref{sec:proof 2}), the final result becomes independent of \( \alpha_{0} \). Whether this threshold can be improved to \( \mathbf{T} = \alpha_{0} \) remains an open problem.

For covariate shift adaptation, Shimodaira \cite{Shimodaira2000ImprovingPI} pioneered importance weighting (IW) to correct distributional bias in parametric regression. Their analysis demonstrates that IW compensates for discrepancies induced by model misspecification. Extensions confirm that IW improves convergence rates for parametric models under misspecification \cite{Huang2006CorrectingSS, Wen2014RobustLU}. However, the role of IW in nonparametric regimes is less clear: empirical studies show that the effects of IW gradually attenuate during training on neural networks \cite{Byrd2019WhatEI}, while theoretical analyses also challenge conventional IW paradigms \cite{Xu2021UnderstandingRI}. Even for well-studied kernel methods, analyses under covariate shift remain sparse. To our knowledge, only Gogolashvili et al. \cite{Gogolashvili2023WhenIW} investigate the scenario where the regression function lies outside the RKHS, but their analysis is restricted to KRR, leaving broader spectral algorithms unaddressed; moreover, their work guarantees convergence only to RKHS projections rather than the regression function itself. Finally, to implement IW in practice, one must estimate the density ratio using unlabeled data. Traditional estimation methods require strict boundedness assumptions \cite{Sugiyama2012DensityRE, Nguyen2024RegularizedRN}, while recent advances relax these restrictions through neural networks \cite{Feng2024DeepNQ, Xu2025EstimatingUD}. It is also worth noting that the standard definition \( w = \dd \rhoTX / \dd \rhoSX \) represents only one particular formulation among various weighting paradigms, as systematically cataloged in \cite{Kimura2024ShortSI}.

Compared to existing results, we extend the work of Fan et al. \cite{Fan2025SpectralAU} in two key directions: (1) we incorporate the embedding index to address model misspecification; (2) we generalize their \( L^2 \)-norm convergence results to the norms of interpolation spaces \( \norm{ \cdot }_{\ranH{\gamma}} \), which encompasses both the RKHS norm and the \( L^2 \)-norm as special cases. Our work also relates to Gizewski et al. \cite{Gizewski2022RegularizationUD}, Ma et al. \cite{Ma2023OptimallyTC}, Gogolashvili et al. \cite{Gogolashvili2023WhenIW}, and Feng et al. \cite{Feng2023TowardsUA}. While these studies focus exclusively on well-specified settings, several critical distinctions emerge:
\begin{itemize}
    \item  
    Gizewski et al. \cite{Gizewski2022RegularizationUD} analyze spectral algorithms under covariate shift,  assuming a uniformly bounded density ratio \( w \) and a well-specified model (i.e., \( \frho \in \calH \)). Furthermore, they introduce a framework for estimating the density ratio \( w \), this estimated ratio is then integrated into the spectral algorithm to produce the final estimator. However, their approach requires the restrictive assumption that \( w \) belongs to the RKHS \( \calH \), which implies uniform boundedness of the density ratio.

    \item Ma et al. \cite{Ma2023OptimallyTC} investigate kernel ridge regression---a special case of spectral algorithms---under covariate shift for the well-specified model (i.e., \( \frho \in \calH \)). They study two cases: 1) with a uniformly bounded density ratio, KRR achieves minimax optimal convergence rates (up to logarithmic factors); 2) when the density ratio is unbounded but has finite second moment, a truncated ratio also yields minimax optimal rates (up to logarithmic factors). Feng et al. \cite{Feng2023TowardsUA} extend this analysis to general loss functions beyond squared loss. However, these analyses are limited by their requirement that the eigenfunctions be uniformly bounded---an assumption that is difficult to verify.

    \item Our moment condition on the density ratio is adopted from Gogolashvili et al. \cite{Gogolashvili2023WhenIW}, who also employ ratio truncation to achieve near-optimal convergence rates when the ratio is unbounded. In our work, we extend their kernel ridge regression framework to broader spectral algorithms and establish convergence guarantees under misspecified settings.
\end{itemize}
For better illustration, a comparison with the most relevant works in kernel methods is summarized in \tabref{tab:comparison}.
\begin{table}[t]
    \small
    \centering
    \caption{Comparison with Existing Works in Kernel Methods}
    \label{tab:comparison}
    \begin{tabular}{>{\raggedright\arraybackslash}p{2.5cm}>{\centering\arraybackslash}p{3.5cm}>{\centering\arraybackslash}p{4.5cm}>{\centering\arraybackslash}p{3cm}}
        \toprule
         & \textbf{Spectral Algorithm} & \textbf{Model Misspecification} & \textbf{Covariate Shift} \\
        \midrule
        \cite{Gerfo2008SpectralAS, Guo2017LearningTD, Blanchard2018OptimalRR, Fan2024SpectralAF}
         & \( \checkmark \)
         &
         &                                                                                       \\
        \cite{Ma2023OptimallyTC, Feng2023TowardsUA, Gogolashvili2023WhenIW}
         &
         &
         & \( \checkmark \)                                                                      \\
        \cite{Lin2020OptimalRS, Zhang2024OptimalityMS}
         & \( \checkmark \)
         & \( \checkmark \)
         &                                                                                       \\
        \cite{Gizewski2022RegularizationUD, Fan2025SpectralAU}
         & \( \checkmark \)
         &
         & \( \checkmark \)                                                                      \\
        \textbf{Ours}
         & \( \checkmark \)
         & \( \checkmark \)
         & \( \checkmark \)                                                                      \\
        \bottomrule
    \end{tabular}
\end{table}

\section{Proof of Main Theorems}
\label{sec:proof}

This section presents the proofs of \thmref{thm:importance weighting} and \thmref{thm:truncated weighting}. Our analysis proceeds in three main steps: first, we decompose the estimation error into distinct components; second, we establish individual bounds for each component through auxiliary propositions; and third, we combine these bounds to complete the overall argument. Auxiliary technical results supporting these propositions are deferred to \hyperref[sec:appendix]{Appendix}.

\subsection{Proof of \thmref{thm:importance weighting}}
\label{sec:proof 1}

We begin the proof with an error decomposition: the excess error \( \norm{ \hflam - \frho }_{\ranH{\gamma}} \) can be decomposed into two distinct components:
\[
    \norm{ \hflam - \frho }_{\ranH{\gamma}}
    \le \underbrace{
    \vphantom{\bigg\|}
    \norm{ \hflam - \flam }_{\ranH{\gamma}} }_{\text{estimation error}}
    + \underbrace{
    \vphantom{\bigg\|}
    \norm{ \flam - \frho }_{\ranH{\gamma}} }_{\text{approximation error}},
\]
where
\[
    \flam = \philam(\LK) \, \LK \, \frho.
\]

The approximation error is bounded by the following proposition:
\begin{proposition}
    \label{pro:approximation error}
    Under \aspref{asp:source} with \( r \in (0, \tau] \), the following inequality holds:
    \[
        \norm{ \flam - \frho }_{\ranH{\gamma}}
        \le F \norm{ \urho }_{\rhoTX} \cdot \lambda^{r - \frac{\gamma}{2}},
        \quad 0 \le \gamma \le \min \family{ 2r, 1 }.
    \]
\end{proposition}
\begin{proof}
    By the definition of \( \flam \), the approximation error can be expressed as:
    \begin{align*}
        \norm{ \flam - \frho }_{\ranH{\gamma}}
         & = \norm{ \LK^{\frac{1 - \gamma}{2}} \, (\flam - \frho) }_{\calH}
        = \norm{ \LK^{\frac{1 - \gamma}{2}} \, (\philam(\LK) \, \LK \, \frho - \frho) }_{\calH}         \\
         & = \norm{ \LK^{\frac{1 - \gamma}{2}} \, ( \, I - \LK \, \philam(\LK) \, ) \, \frho }_{\calH}.
    \end{align*}
    Applying \aspref{asp:source}, which assumes \( \frho = \LK^{r} \, \urho \) with \( r \in (0, \tau] \) and \( \urho \in L^{2}(\calX, \rhoTX) \), we obtain:
    \begin{align*}
         & \eqspace \norm{ \LK^{\frac{1 - \gamma}{2}} \, ( \, I - \LK \, \philam(\LK) \, ) \, \frho }_{\calH}                     \\
         & = \norm{ \LK^{\frac{1 - \gamma}{2}} \, ( \, I - \LK \, \philam(\LK) \, ) \, \LK^{r} \, \urho }_{\calH}
        = \norm{ \LK^{r - \frac{\gamma}{2}} \, ( \, I - \LK \, \philam(\LK) \, ) \, \LK^{1/2} \, \urho }_{\calH}                  \\
         & \le \norm{ \LK^{r - \frac{\gamma}{2}} \, ( \, I - \LK \, \philam(\LK) \, ) } \cdot \norm{ \LK^{1/2} \, \urho }_{\calH}
        = \norm{ \LK^{r - \frac{\gamma}{2}} \, ( \, I - \LK \, \philam(\LK) \, ) } \cdot \norm{ \urho }_{\rhoTX},
    \end{align*}
    where \( \norm{ \cdot } \) denotes the operator norm on \( \calH \). Using the filter function property \eqref{eq:psilam} and \( 0 \le \gamma \le \min \family{ 2r, 1 } \), we bound the operator norm term:
    \[
        \norm{ \LK^{r - \frac{\gamma}{2}} \, ( \, I - \LK \, \philam(\LK) \, ) }
        \le \sup_{t \in [0, \kappa^{2}]} t^{r - \frac{\gamma}{2}} |1 - t \philam(t)|
        \le F \lambda^{r - \frac{\gamma}{2}}.
    \]
    Combining these results yields the desired bound:
    \[
        \norm{ \flam - \frho }_{\ranH{\gamma}} \le F \norm{ \urho }_{\rhoTX} \cdot \lambda^{r - \frac{\gamma}{2}}. \qedhere
    \]
\end{proof}

As shown in \proref{pro:approximation error}, the approximation error converges at the rate \( O(\lambda^{r - \frac{\gamma}{2}}) \). The following proposition establishes that the estimation error \( \norm{ \hflam - \flam }_{\ranH{\gamma}} \) decays at the same rate under appropriate conditions.

\begin{proposition}
    \label{pro:estimation error}
    Suppose that \aspref{asp:ratio} holds with \( p \in [1, \infty] \), \aspref{asp:source} holds with \( r \in (0, \tau] \), \aspref{asp:eigenvalue} holds with \( \beta > 1 \), and \aspref{asp:embedding} holds with \( \alpha_{0} \in [1/\beta, 1) \). Let \( \lambda = n^{-s} \), where the exponent \( s \) satisfies:
    \begin{enumerate}
        \item \( s \cdot \max \family{ \alpha, \frac{\alpha}{2} + r, \alpha + \frac{1}{\beta} + \frac{\alpha - 1/\beta}{p}, 2 r + \frac{1}{\beta} + \frac{\alpha - 1/\beta}{p} } < 1 \); \hfill \label{eq:R1} \textup{(\textbf{R1})}\medskip

        \item \( s \le
              \begin{cases}
                  1/2,              & r \in ( 1, 3/2 ]\textit{;}
                  \smallskip                                     \\
                  \frac{1}{2r - 1}, & r > 3/2\textit{.}
              \end{cases} \) \hfill \label{eq:R2} \textup{(\textbf{R2})}
    \end{enumerate}
    The parameter \( \alpha \) is chosen as follows: if \( 2r \le \alpha_{0} \), then \( \alpha \in (\alpha_{0}, 1] \) is arbitrary; if \( 2r > \alpha_{0} \), then \( \alpha \in (\alpha_{0}, \min \family{ 2r, 1 }] \). Then, for any \( \delta \in (0, 1) \) and
    \begin{equation}
        \label{eq:S1} \tag{\textbf{S1}}
        \begin{aligned}
            n \ge \max \biggg\{ & \LE( 16 L M_{\alpha}^{2} \log \frac{6}{\delta} \RI)^{\frac{1}{1 - s \alpha}},                                                                                               \\
                                & \LE( 16 \sigma M_{\alpha}^{1 + \frac{1}{p}} C_{\calN} \log \frac{6}{\delta} \RI)^{\frac{2}{1 - s \LE( \alpha + \frac{1}{\beta} + \frac{\alpha - 1/\beta}{p} \RI)}} \biggg\}
        \end{aligned}
    \end{equation}
    with \( M_{\alpha} = \norm{ \ranH{\alpha} \hookrightarrow L^{\infty}(\calX, \rhoTX) } \), the following bound holds with probability at least \( 1 - \delta \):
    \[
        \norm{ \hflam - \flam }_{\ranH{\gamma}}
        = O \LE( \lambda^{r - \frac{\gamma}{2}} \log \frac{6}{\delta} \RI),
        \quad 0 \le \gamma \le \min \family{ 2r, 1 }.
    \]
\end{proposition}
To prove \proref{pro:estimation error}, we decompose the estimation error \( \norm{ \hflam - \flam }_{\ranH{\gamma}} \) into several components. For notational convenience, define \( \LKlam = \LK + \lambda I \) and \( \hLKlam = \hLK + \lambda I \). The estimation error can be written as:
\begin{equation}
    \label{eq:decomp 1}
    \begin{aligned}
        \norm{ \hflam - \flam }_{\ranH{\gamma}}
         & = \norm{ \LK^{\frac{1 - \gamma}{2}} \, (\hflam - \flam) }_{\calH} \\
         & = \norm{
            \LK^{\frac{1 - \gamma}{2}} \, \LKlam^{-1/2}
            \circ \LKlam^{1/2} \, \hLKlam^{-1/2}
            \circ \hLKlam^{1/2} \, (\hflam - \flam)
        }_{\calH}                                                            \\
         & \le \norm{ \LK^{\frac{1 - \gamma}{2}} \, \LKlam^{-1/2} }
        \cdot \norm{ \LKlam^{1/2} \, \hLKlam^{-1/2} }
        \cdot \norm{ \hLKlam^{1/2} \, (\hflam - \flam) }_{\calH},
    \end{aligned}
\end{equation}
where \( \circ \) denotes operator composition. Furthermore, by the definition of \( \hflam \), the third term \( \norm{ \hLKlam^{1/2} \, (\hflam - \flam) }_{\calH} \) in \eqref{eq:decomp 1} can be decomposed as
\begin{equation}
    \label{eq:decomp 2}
    \begin{aligned}
         & \eqspace \norm{ \hLKlam^{1/2} \, (\hflam - \flam) }_{\calH}                                                                                         \\
         & = \norm{ \hLKlam^{1/2} \, (\philam(\hLK) \, \hSKa \, \bsy - \flam) }_{\calH}                                                                        \\
         & = \norm{ \hLKlam^{1/2} \, (\philam(\hLK) \, \hSKa \, \bsy - \LE( \hLK \, \philam(\hLK) + \LE( I - \hLK \, \philam(\hLK) \RI) \RI) \flam ) }_{\calH} \\
         & \le \norm{ \hLKlam^{1/2} \, \philam(\hLK) \, (\hSKa \, \bsy - \hLK \, \flam) }_{\calH}
        + \norm{ \hLKlam^{1/2} \LE( I - \hLK \, \philam(\hLK) \RI) \flam }_{\calH}.
    \end{aligned}
\end{equation}
Combining \eqref{eq:decomp 1} and \eqref{eq:decomp 2} yields the overall error bound:
\begin{equation}
    \label{eq:J}
    \begin{aligned}
        \norm{ \hflam - \flam }_{\ranH{\gamma}}
         & \le \norm{ \LK^{\frac{1 - \gamma}{2}} \, \LKlam^{-1/2} }
        \cdot \norm{ \LKlam^{1/2} \, \hLKlam^{-1/2} }                                                             \\
         & \eqspace \cdot \LE( \norm{ \hLKlam^{1/2} \, \philam(\hLK) \, (\hSKa \, \bsy - \hLK \, \flam) }_{\calH}
        + \norm{ \hLKlam^{1/2} \LE( I - \hLK \, \philam(\hLK) \RI) \flam }_{\calH} \RI)                           \\
         & =J_{1} \cdot J_{2} \cdot (J_{3} + J_{4}),
    \end{aligned}
\end{equation}
where:
\begin{alignat*}{3}
    J_{1} & = \norm{ \LK^{\frac{1 - \gamma}{2}} \, \LKlam^{-1/2} },
          &                                                                                       & \quad J_{2} &  & = \norm{ \LKlam^{1/2} \, \hLKlam^{-1/2} },                                  \\
    J_{3} & = \norm{ \hLKlam^{1/2} \, \philam(\hLK) \, (\hSKa \, \bsy - \hLK \, \flam) }_{\calH},
          &                                                                                       & \quad J_{4} &  & = \norm{ \hLKlam^{1/2} \LE( I - \hLK \, \philam(\hLK) \RI) \flam }_{\calH}.
\end{alignat*}

We bound the estimation error \( \norm{ \hflam - \flam }_{\ranH{\gamma}} \) by separately estimating the terms \( J_{1} \), \( J_{2} \), \( J_{3} \) and \( J_{4} \) in \eqref{eq:J}. The term \( J_{1} = \norm{ \LK^{\frac{1 - \gamma}{2}} \, \LKlam^{-1/2} } \) is bounded using \lemref{lem:cordes} and \lemref{lem:sup fraction}. The bound for the term \( J_{2} = \norm{ \LKlam^{1/2} \, \hLKlam^{-1/2} } \) in \eqref{eq:J} is established by the following proposition.
\begin{proposition}
    \label{pro:LKlam^1/2 hLKlam^-1/2}
    Suppose that \aspref{asp:ratio} holds with \( p \in [1, \infty] \), and \( \calH \) has embedding index \( \alpha_{0} < 1 \). For any \( \alpha \in (\alpha_{0}, 1] \) and \( \delta \in (0, 1) \), if \( n \) and \( \lambda \) satisfy
    \[
        4 \LE( \frac{\tilde{L}_{1}}{n} + \frac{\tilde{\sigma}_{1}}{\sqrt{n}} \RI) \log \frac{6}{\delta}
        \le \frac{1}{2},
    \]
    where
    \[
        \tilde{L}_{1}
        = L M_{\alpha}^{2} \cdot \lambda^{-\alpha},
        \quad \tilde{\sigma}_{1}
        = \sigma M_{\alpha}^{1 + \frac{1}{p}} \cdot \lambda^{-\frac{1 + 1/p}{2} \alpha} \calN^{\frac{1 - 1/p}{2}}(\lambda),
        \quad M_{\alpha}
        = \norm{ \ranH{\alpha} \hookrightarrow L^{\infty}(\calX, \rhoTX) },
    \]
    then with probability at least \( 1 - \delta/3 \), we have
    \[
        J_{2} = \norm{ \LKlam^{1/2} \, \hLKlam^{-1/2} } \le \sqrt{2}.
    \]
\end{proposition}
\begin{proof}
    By \lemref{lem:LKlam^-1/2 (LK - hLK) LKlam^-1/2}, with probability at least \( 1 - \delta/3 \), we have
    \[
        \norm{ \LKlam^{-1/2} \, (\LK - \hLK) \, \LKlam^{-1/2} } \le \frac{1}{2}.
    \]
    Using the decomposition
    \[
        \hLKlam = \hLK + \lambda = (\hLK - \LK) + (\LK + \lambda) = (\hLK - \LK) + \LKlam,
    \]
    we proceed as follows:
    \begin{align*}
        J_{2}^{2} & = \norm{ \LKlam^{1/2} \, \hLKlam^{-1/2} }^{2}
        = \norm{ \LKlam^{1/2} \, \hLKlam^{-1} \, \LKlam^{1/2} }
        = \norm{ \LE( \LKlam^{-1/2} \, \hLKlam \, \LKlam^{-1/2} \RI)^{-1} }                      \\
                  & = \norm{ \LE( I - \LKlam^{-1/2} \, (\LK - \hLK) \, \LKlam^{-1/2} \RI)^{-1} }
        \le \sum_{k=0}^{\infty} \norm{ \LKlam^{-1/2} \, (\LK - \hLK) \, \LKlam^{-1/2} }^{k}      \\
                  & \le 2. \qedhere
    \end{align*}
\end{proof}

The next proposition provides a bound for the term \( J_{3} = \norm{ \hLKlam^{1/2} \, \philam(\hLK) \, (\hSKa \, \bsy - \hLK \, \flam) }_{\calH} \) in \eqref{eq:J}.
\begin{proposition}
    \label{pro:part 1}
    Suppose that \aspref{asp:ratio} holds with \( p \in [1, \infty] \), \aspref{asp:source} holds with \( r \in (0, \tau] \), \aspref{asp:eigenvalue} holds with \( \beta > 1 \), and \aspref{asp:embedding} holds with \( \alpha_{0} \in [1/\beta, 1) \). Let \( \lambda = n^{-s} \), where \( s \) satisfies \neweqref{eq:R1}{R1} with \( \alpha \in (\alpha_{0}, 1] \) if \( 2r \le \alpha_{0} \), and \( \alpha_{0} < \alpha \le \min \family{ 2r, 1 } \) if \( 2r > \alpha_{0} \). Then, for any \( \delta \in (0, 1) \) and sufficiently large \( n \) satisfying \eqref{eq:S1},
    \[
        J_{3} = \norm{ \hLKlam^{1/2} \, \philam(\hLK) \, (\hSKa \, \bsy - \hLK \, \flam) }_{\calH} = O \LE( \lambda^{r} \log \frac{6}{\delta} \RI)
    \]
    holds with probability at least \( 1 - (2 \delta) / 3 \).
\end{proposition}
\begin{proof}
    We begin by decomposing the target norm:
    \begin{equation}
        \label{eq:decomp J3}
        \begin{aligned}
            J_{3} & =\norm{ \hLKlam^{1/2} \, \philam(\hLK) \, (\hSKa \, \bsy - \hLK \, \flam) }_{\calH} \\
                  & = \norm{
                \hLKlam^{1/2} \, \philam(\hLK) \, \hLKlam^{1/2}
                \circ \hLKlam^{-1/2} \, \LKlam^{1/2}
                \circ \LKlam^{-1/2} \, (\hSKa \, \bsy - \hLK \, \flam)
            }_{\calH}                                                                                   \\
                  & \le \norm{ \hLKlam^{1/2} \, \philam(\hLK) \, \hLKlam^{1/2} }
            \cdot \norm{ \hLKlam^{-1/2} \, \LKlam^{1/2} }
            \cdot \norm{ \LKlam^{-1/2} \, (\hSKa \, \bsy - \hLK \, \flam) }_{\calH}.
        \end{aligned}
    \end{equation}

    For the first term, \( \norm{ \hLKlam^{1/2} \, \philam(\hLK) \, \hLKlam^{1/2} } \), we use the filter function property \eqref{eq:philam} with \( \theta = 0 \) and \( \theta = 1 \):
    \begin{align*}
        \norm{ \hLKlam^{1/2} \, \philam(\hLK) \, \hLKlam^{1/2} }
         & = \norm{ \hLKlam \, \philam(\hLK) }
        \le \norm{ \hLK \, \philam(\hLK) } + \lambda \cdot \norm{ \philam(\hLK)}                                     \\
         & \le \sup_{t \in [0, \kappa^{2}]} |t \philam(t)| + \lambda \cdot \sup_{t \in [0, \kappa^{2}]} |\philam(t)| \\
         & \le 2 E.
    \end{align*}

    For the second term, \( \norm{ \hLKlam^{-1/2} \, \LKlam^{1/2} } \), under conditions \eqref{eq:S1} and \neweqref{eq:R1}{R1}, \proref{pro:LKlam^1/2 hLKlam^-1/2} implies
    \[
        \norm{ \hLKlam^{-1/2} \, \LKlam^{1/2} } \le \sqrt{2}
    \]
    with probability at least \( 1 - \delta/3 \).

    For the third term, \( \norm{ \LKlam^{-1/2} \, (\hSKa \, \bsy - \hLK \, \flam) }_{\calH} \), we decompose it by adding and subtracting its expectation:
    \begin{equation}
        \label{eq:decomp 3}
        \begin{aligned}
             & \eqspace \norm{ \LKlam^{-1/2} \, (\hSKa \, \bsy - \hLK \, \flam) }_{\calH}                                   \\
             & \le \norm{ \LKlam^{-1/2} \LE( (\hSKa \, \bsy - \hLK \, \flam) - (\LK \, \frho - \LK \, \flam) \RI) }_{\calH}
            + \norm{ \LKlam^{-1/2} \, (\LK \, \frho - \LK \, \flam) }_{\calH}.
        \end{aligned}
    \end{equation}

    To bound the first component in \eqref{eq:decomp 3}, we define the point evaluation operator \( \Kx \) and its adjoint \( \Kxa \) as
    \begin{equation}
        \label{eq:Kx}
        \begin{alignedat}{2}
             & \Kx \colon  \calH \to \bbR, \quad &  & f \mapsto \inner{ K(\cdot, x), f }_{\calH}; \\
             & \Kxa \colon \bbR \to \calH, \quad &  & y \mapsto y \, K(\cdot, x).
        \end{alignedat}
    \end{equation}
    Then \( \LK = \E{ w(x) \, \Kx \Kxa } \). Define \( \xi = \xi(z) = \LKlam^{-1/2} \, w(x) \, (\Kx \, y - \Kx \Kxa \, \flam ) \), so that we aim to bound
    \[
        \norm{ \frac{1}{n} \sum_{i=1}^{n} \xi_{i} - \E{ \xi } }_{\calH},
    \]
    where \( \xi_{i} = \xi(x_{i}) \). Rewriting \( \xi \) gives
    \[
        \xi = \LKlam^{-1/2} \, w(x) \, \Kx \, ( y - \flam(x) )
        = \LKlam^{-1/2} \, K(\cdot, x) \cdot w(x) \, (y - \flam(x)).
    \]
    Through the following steps, we establish a uniform bound for \( |y - \flam(x)| \):
    \begin{enumerate}
        \item By \aspref{asp:source},
              \[
                  \norm{ \frho }_{\infty} \le G,
                  \quad |y| \le G.
              \]

        \item For \( 2r \le \alpha_{0} \) and any \( \alpha \in (\alpha_{0}, 1] \),
              \begin{align*}
                  \norm{ \flam }_{\infty}
                   & \le M_{\alpha} \norm{ \flam }_{\ranH{\alpha}}
                  = M_{\alpha} \norm{ \philam(\LK) \, \LK^{r+1} \, \urho }_{\ranH{\alpha}}                                                   \\
                   & = M_{\alpha} \norm{ \LK^{\frac{1 - \alpha}{2}} \, \philam(\LK) \, \LK^{r+1} \, \urho }_{\calH}
                  \le M_{\alpha} \norm{ \LK^{1 - \LE( \frac{\alpha}{2} - r \RI)} \, \philam(\LK) } \cdot \norm{ \LK^{1/2} \, \urho }_{\calH} \\
                   & \le M_{\alpha} E \norm{ \urho }_{\rhoTX} \cdot \lambda^{-\LE( \frac{\alpha}{2} - r \RI)},
              \end{align*}
              where \( M_{\alpha} = \norm{ \ranH{\alpha} \hookrightarrow L^{\infty}(\calX, \rhoTX) } \), and the last inequality uses the filter function property \eqref{eq:philam}.

        \item When \( 2r > \alpha_{0} \), for \( \alpha_{0} < \alpha \le \min \family{ 2r, 1 } \), the inclusion \( \frho \in \ranH{2r} \hookrightarrow \ranH{\alpha} \) holds. Applying \proref{pro:approximation error} with \( \gamma = \alpha \) yields
              \[
                  \norm{ \frho - \flam }_{\infty}
                  \le M_{\alpha} \norm{ \frho - \flam }_{\ranH{\alpha}}
                  \le M_{\alpha} F \norm{ \urho }_{\rhoTX} \cdot \lambda^{-\LE( \frac{\alpha}{2} - r \RI)}.
              \]

        \item Combining these results, we obtain
              \begin{equation}
                  \label{eq:y - flam(x)}
                  |y - \flam(x)|
                  \le M_{\alpha} (E + F) \norm{ \urho }_{\rhoTX} \cdot \lambda^{-\LE( \frac{\alpha}{2} - r \RI)} + 2 G,
                  \quad \rhoTX \text{-a.e. } x \in \calX.
              \end{equation}
    \end{enumerate}
    This leads to
    \begin{align*}
         & \eqspace \E{ \norm{ \xi }_{\calH}^{m} }                                                                                                                                           \\
         & \le \LE( M_{\alpha} (E + F) \norm{ \urho }_{\rhoTX} \cdot \lambda^{-\LE( \frac{\alpha}{2} - r \RI)} + 2 G \RI)^{m}
        \cdot \int_{\calX} \norm{ \LKlam^{-1/2} \, K(\cdot, x) }_{\calH}^{m} w^{m-1}(x) \, \dd \rhoTx                                                                                        \\
         & \le \LE( M_{\alpha} (E + F) \norm{ \urho }_{\rhoTX} \cdot \lambda^{-\LE( \frac{\alpha}{2} - r \RI)} + 2 G \RI)^{m} \cdot \LE( \int_{\calX} w^{p(m-1)}(x) \, \dd \rhoTx \RI)^{1/p} \\
         & \eqspace \cdot \LE( \int_{\calX} \norm{ \LKlam^{-1/2} \, K(\cdot, x) }_{\calH}^{qm} \, \dd \rhoTx \RI)^{1/q}                                                                      \\
         & \le \LE( M_{\alpha} (E + F) \norm{ \urho }_{\rhoTX} \cdot \lambda^{-\LE( \frac{\alpha}{2} - r \RI)} + 2 G \RI)^{m}
        \cdot \frac{1}{2} m! L^{m-2} \sigma^{2}
        \cdot \LE( \int_{\calX} \norm{ \LKlam^{-1/2} \, K(\cdot, x) }_{\calH}^{qm} \, \dd \rhoTx \RI)^{1/q},
    \end{align*}
    where \( 1/p + 1/q = 1 \). Following the argument in the proof of \lemref{lem:LKlam^-1/2 (LK - hLK) LKlam^-1/2}, we have
    \[
        \LE( \int_{\calX} \norm{ \LKlam^{-1/2} \, K(\cdot, x) }_{\calH}^{qm} \, \dd \rhoTx \RI)^{1/q}
        \le \LE( \LE( M_{\alpha} \lambda^{-\alpha/2} \RI)^{qm - 2} \calN(\lambda) \RI)^{1/q}.
    \]
    Combining these results yields
    \[
        \E{ \norm{ \xi }_{\calH}^{m} }
        \le \frac{1}{2} m! \tilde{L}_{2}^{m-2} \tilde{\sigma}_{2}^{2}
    \]
    with
    \begin{align*}
        \tilde{L}_{2}
         & = L M_{\alpha} \LE( M_{\alpha} (E + F) \norm{ \urho }_{\rhoTX} \cdot \lambda^{-\LE( \frac{\alpha}{2} - r \RI)} + 2 G \RI) \lambda^{-\alpha/2},                                                                                        \\
        \tilde{\sigma}_{2}
         & = \sigma M_{\alpha}^{1 - \frac{1}{q}} \LE( M_{\alpha} (E + F) \norm{ \urho }_{\rhoTX} \cdot \lambda^{-\LE( \frac{\alpha}{2} - r \RI)} + 2 G \RI) \lambda^{-\frac{\alpha}{2} \LE( 1 - \frac{1}{q} \RI)} \calN^{\frac{1}{2q}}(\lambda).
    \end{align*}
    Applying \lemref{lem:bernstein}, we conclude that with probability exceeding \( 1 - \delta/3 \),
    \begin{equation}
        \label{eq:rate 1}
        \norm{ \LKlam^{-1/2} \LE( (\hSKa \, \bsy - \hLK \, \flam) - (g - \LK \, \flam) \RI) }_{\calH}
        \le 4 \LE( \frac{\tilde{L}_{2}}{n} + \frac{\tilde{\sigma}_{2}}{\sqrt{n}} \RI) \log \frac{6}{\delta}.
    \end{equation}
    The rate of \eqref{eq:rate 1} simplifies to:
    \begin{align*}
         & \eqspace \LE( \frac{\lambda^{r - \alpha} + \lambda^{-\alpha/2}}{n}
        + \frac{\LE( \lambda^{-\LE( \frac{\alpha}{2} - r \RI)} + 1 \RI) \lambda^{-\frac{\alpha}{2p}} \calN^{\frac{1}{2} - \frac{1}{2p}}(\lambda)}{\sqrt{n}} \RI) \log \frac{6}{\delta} \\
         & \asymp \LE(
        \frac{n^{s \alpha} + n^{s \LE( \frac{\alpha}{2} + r \RI)}}{n}
        + \LE( \frac{n^{s \LE( \alpha + \frac{1}{\beta} + \frac{\alpha - 1/\beta}{p} \RI)} + n^{s \LE( 2 r + \frac{1}{\beta} + \frac{\alpha - 1/\beta}{p} \RI)}}{n} \RI)^{1/2} \RI) \lambda^{r} \log \frac{6}{\delta}.
    \end{align*}
    Under \neweqref{eq:R1}{R1}, this decays as \( O( \, \lambda^{r} \log (6/\delta) \, ) \).

    For the second component \( \norm{ \LKlam^{-1/2} \, (\LK \, \frho - \LK \, \flam) }_{\calH} \) in \eqref{eq:decomp 3}, we apply \proref{pro:approximation error} with \( \gamma = 0 \):
    \begin{equation}
        \begin{aligned}
            \norm{ \LKlam^{-1/2} \, (\LK \, \frho - \LK \, \flam) }_{\calH}
             & = \norm{ \LKlam^{-1/2} \, \LK^{1/2} \circ \LK^{1/2} \, (\frho - \flam) }_{\calH} \\
             & \le \norm{ \LKlam^{-1/2} \, \LK^{1/2} }
            \cdot \norm{ \frho - \flam }_{\rhoTX}                                               \\
             & \le F \norm{ \urho }_{\rhoTX} \cdot \lambda^{r},
        \end{aligned}
    \end{equation}
    which is also \( O(\lambda^{r}) \).

    In summary, \( J_{3} = \norm{ \hLKlam^{1/2} \, \philam(\hLK) \, (\hSKa \, \bsy - \hLK \, \flam) }_{\calH} = O( \, \lambda^{r} \log (6/\delta) \, ) \), which completes the proof.
\end{proof}

Finally, we bound the term \( J_{4} = \norm{ \hLKlam^{1/2} \LE( I - \hLK \, \philam(\hLK) \RI) \flam }_{\calH} \) from \eqref{eq:J} as follows.
\begin{proposition}
    \label{pro:hLKlam^1/2 psilam(hLKlam) flam}
    Suppose that \aspref{asp:source} holds with \( r \in (0, \tau] \), and assume the conditions of \proref{pro:LKlam^1/2 hLKlam^-1/2}. Then for any \( r \in (0, \tau] \) and \( \delta \in (0, 1) \), with probability at least \( 1 - (2 \delta) / 3 \),
    \begin{align*}
        J_{4} & = \norm{ \hLKlam^{1/2} \LE( I - \hLK \, \philam(\hLK) \RI) \flam }_{\calH} \\
              & \le 2 \sqrt{2} E F \norm{ \urho }_{\rhoTX} \cdot \LE(
        \lambda^{r}
        + \Delta \cdot \lambda^{1/2} n^{-\frac{\min \family{ 2r, 3 } - 1}{4}} \log \frac{6}{\delta}
        \cdot \ind{r > 1}
        \RI),
    \end{align*}
    where
    \[
        \Delta = 4 r \kappa^{2r-1} (L + \sigma)
    \]
    is a constant independent of \( n \) and \( \delta \).
\end{proposition}
\begin{proof}
    The analysis is divided into three cases based on the source parameter \( r \).
    \begin{itemize}
        \item \( 0 < r < 1/2 \): Starting from the expansion:
              \begin{equation}
                  \label{eq:hLKlam psilam(hLKlam) flam}
                  \begin{aligned}
                      J_{4} & = \norm{ \hLKlam^{1/2} \LE( I - \hLK \, \philam(\hLK) \RI) \flam }_{\calH}
                      = \norm{ \hLKlam^{1/2} \LE( I - \hLK \, \philam(\hLK) \RI) \philam(\LK) \, \LK \, \frho }_{\calH}                      \\
                            & = \norm{ \hLKlam^{1/2} \LE( I - \hLK \, \philam(\hLK) \RI) \circ \philam(\LK) \, \LK^{r+1} \, \urho }_{\calH}.
                  \end{aligned}
              \end{equation}
              By the inequality \( (a + b)^{1/2} \le a^{1/2} + b^{1/2} \), we obtain:
              \begin{align*}
                  \norm{ \hLKlam^{1/2} \LE( I - \hLK \, \philam(\hLK) \RI) }
                   & \le \sup_{t \in [0, \kappa^{2}]} (t + \lambda)^{1/2} |1 - t \philam(t)|                                                           \\
                   & \le \sup_{t \in [0, \kappa^{2}]} t^{1/2} |1 - t \philam(t)| + \lambda^{1/2} \cdot \sup_{t \in [0, \kappa^{2}]} |1 - t \philam(t)| \\
                   & \le F \lambda^{1/2} + \lambda^{1/2} \cdot F
                  = 2 F \lambda^{1/2}.
              \end{align*}
              The remaining term is bounded by:
              \begin{align*}
                  \norm{ \philam(\LK) \, \LK^{r+1} \, \urho }_{\calH}
                   & = \norm{ \LK^{r + \frac{1}{2}} \, \philam(\LK) \, \LK^{1/2} \, \urho }_{\rhoTX}
                  \le \norm{ \LK^{r + \frac{1}{2}} \, \philam(\LK) } \cdot \norm{ \urho }_{\rhoTX}   \\
                   & \le E \norm{ \urho }_{\rhoTX} \cdot \lambda^{r - \frac{1}{2}}.
              \end{align*}

              Combining these estimates yields:
              \begin{equation}
                  \label{eq:case1}
                  J_{4}
                  \le 2 E F \norm{ \urho }_{\rhoTX} \cdot \lambda^{r}.
              \end{equation}

        \item \( 1/2 \le r \le 1 \): Using the expansion in \eqref{eq:hLKlam psilam(hLKlam) flam}:
              \begin{align*}
                  J_{4} & = \norm{ \hLKlam^{1/2} \LE( I - \hLK \, \philam(\hLK) \RI) \flam }_{\calH}
                  = \norm{ \hLKlam^{1/2} \LE( I - \hLK \, \philam(\hLK) \RI) \philam(\LK) \, \LK^{r+1} \, \urho }_{\calH}                                     \\
                        & \le \norm{ \hLKlam^{1/2} \LE( I - \hLK \, \philam(\hLK) \RI) \philam(\LK) \, \LK^{r + \frac{1}{2}} } \cdot \norm{ \urho }_{\rhoTX}.
              \end{align*}
              Due to the constraint \( \theta \in [0, 1] \) in \eqref{eq:philam}, we decompose the operator norm as:
              \begin{align*}
                   & \eqspace \norm{ \hLKlam^{1/2} \LE( I - \hLK \, \philam(\hLK) \RI) \philam(\LK) \, \LK^{r + \frac{1}{2}} } \\
                   & = \norm{
                      \hLKlam^{1/2} \LE( I - \hLK \, \philam(\hLK) \RI) \hLKlam^{r - \frac{1}{2}}
                      \circ \hLKlam^{-\LE( r - \frac{1}{2} \RI)} \, \LKlam^{r - \frac{1}{2}}
                      \circ \LKlam^{-\LE( r - \frac{1}{2} \RI)} \, \LK^{r - \frac{1}{2}}
                  \circ \LK \, \philam(\LK) }                                                                                  \\
                   & \le \norm{ \hLKlam^{r} \LE( I - \hLK \, \philam(\hLK) \RI) }
                  \cdot \norm{ \hLKlam^{-\LE( r - \frac{1}{2} \RI)} \, \LKlam^{r - \frac{1}{2}} }
                  \cdot \norm{ \LKlam^{-\LE( r - \frac{1}{2} \RI)} \, \LK^{r - \frac{1}{2}} }
                  \cdot \norm{ \LK \, \philam(\LK) }.
              \end{align*}
              We bound each factor:
              \begin{enumerate}[i.]
                  \item \( \norm{ \hLKlam^{r} \LE( I - \hLK \, \philam(\hLK) \RI) } \le 2 F \lambda^{r} \);

                  \item By \lemref{lem:cordes} and \proref{pro:LKlam^1/2 hLKlam^-1/2}, with probability at least \( 1 - \delta / 3 \),
                        \[
                            \norm{ \hLKlam^{-\LE( r - \frac{1}{2} \RI)} \, \LKlam^{r - \frac{1}{2}} } \le \norm{ \hLKlam^{-1/2} \, \LKlam^{1/2} }^{2r - 1} \le 2^{r - \frac{1}{2}} \le \sqrt{2};
                        \]

                  \item Using \lemref{lem:cordes} again,
                        \[
                            \norm{ \LKlam^{-\LE( r - \frac{1}{2} \RI)} \, \LK^{r - \frac{1}{2}} } \le \norm{ \LKlam^{-1} \, \LK }^{r - \frac{1}{2}} \le 1;
                        \]

                  \item \( \norm{ \LK \, \philam(\LK) } \le E \).
              \end{enumerate}

              Combining these bounds gives:
              \begin{equation}
                  \label{eq:case2}
                  J_{4}
                  \le 2 \sqrt{2} E F \norm{ \urho }_{\rhoTX} \cdot \lambda^{r}.
              \end{equation}

        \item \( r > 1 \): Since \( \theta \in [0, \tau] \) in \eqref{eq:psilam}, we employ a different decomposition. Starting from:
              \begin{align*}
                  J_{4} & = \norm{ \hLKlam^{1/2} \LE( I - \hLK \, \philam(\hLK) \RI) \flam }_{\calH}
                  \le \norm{ \hLKlam^{1/2} \LE( I - \hLK \, \philam(\hLK) \RI) \philam(\LK) \, \LK^{r + \frac{1}{2}} } \cdot \norm{ \urho }_{\rhoTX}  \\
                        & \le \norm{ \hLKlam^{1/2} \LE( I - \hLK \, \philam(\hLK) \RI) \LK^{r - \frac{1}{2}} } \cdot E \cdot \norm{ \urho }_{\rhoTX},
              \end{align*}
              we write:
              \begin{align*}
                   & \eqspace \norm{ \hLKlam^{1/2} \LE( I - \hLK \, \philam(\hLK) \RI) \LK^{r - \frac{1}{2}} }                                                         \\
                   & = \norm{ \hLKlam^{1/2} \LE( I - \hLK \, \philam(\hLK) \RI) \LE( (\LK^{r - \frac{1}{2}} - \hLK^{r - \frac{1}{2}}) + \hLK^{r - \frac{1}{2}}) \RI) } \\
                   & \le \norm{ \hLKlam^{1/2} \LE( I - \hLK \, \philam(\hLK) \RI) } \cdot \norm{ \LK^{r - \frac{1}{2}} - \hLK^{r - \frac{1}{2}} }
                  + \norm{ \hLKlam^{1/2} \LE( I - \hLK \, \philam(\hLK) \RI) \hLK^{r - \frac{1}{2}} }.
              \end{align*}

              Bounding the first and third terms:
              \begin{enumerate}[i.]
                  \item \( \norm{ \hLKlam^{1/2} \LE( I - \hLK \, \philam(\hLK) \RI) } \le 2 F \lambda^{1/2} \).

                  \item Applying the inequality \( (a + b)^{1/2} \le a^{1/2} + b^{1/2} \),
                        \begin{align*}
                            \norm{ \hLKlam^{1/2} \LE( I - \hLK \, \philam(\hLK) \RI) \hLK^{r - \frac{1}{2}} }
                             & \le \sup_{t \in [0, \kappa^{2}]} (t + \lambda)^{1/2} |1 - t \philam(t)| t^{r - \frac{1}{2}}                                                         \\
                             & \le \sup_{t \in [0, \kappa^{2}]} t^{r} |1 - t \philam(t)| + \lambda^{1/2} \cdot \sup_{t \in [0, \kappa^{2}]} t^{r - \frac{1}{2}} |1 - t \philam(t)| \\
                             & \le 2 F \lambda^{r}.
                        \end{align*}
              \end{enumerate}

              For the second term \( \norm{ \LK^{r - \frac{1}{2}} - \hLK^{r - \frac{1}{2}} } \), we invoke \lemref{lem:A^s - B^s}:
              \begin{align*}
                  \norm{ \LK^{r - \frac{1}{2}} - \hLK^{r - \frac{1}{2}} } \le
                  \begin{cases}
                      \norm{ \LK - \hLK }^{r - \frac{1}{2}},                       & r \in (1, 3/2];
                      \medskip                                                                       \\
                      \LE( r - \frac{1}{2} \RI) \kappa^{2r-3} \norm{ \LK - \hLK }, & r > 3/2.
                  \end{cases}
              \end{align*}
              By \lemref{lem:hLK - LK}, with probability at least \( 1 - \delta/3 \),
              \[
                  \norm{ \hLK - \LK }
                  \le 4 \kappa^{2} \LE( \frac{L}{n} + \frac{\sigma}{\sqrt{n}} \RI) \log \frac{6}{\delta}
                  \le 4 \kappa^{2} (L + \sigma) n^{-1/2} \log \frac{6}{\delta}.
              \]
              Substituting this bound, we obtain with probability at least \( 1 - \delta/3 \),
              \begin{equation}
                  \label{eq:case3}
                  J_{4}
                  \le 2 E F \norm{ \urho }_{\rhoTX} \cdot \LE(
                  \lambda^{r}
                  + \Delta \cdot \lambda^{1/2} n^{-\frac{\min \family{ 2r, 3 } - 1}{4}} \log \frac{6}{\delta}
                  \RI),
              \end{equation}
              where
              \[
                  \Delta = 4 r \kappa^{2r-1} (L + \sigma).
              \]
    \end{itemize}
    Then the proof is complete by by combining \eqref{eq:case1}, \eqref{eq:case2}, and \eqref{eq:case3}.
\end{proof}
Now we are in a position to prove \proref{pro:estimation error}.

\begin{proofof}{\proref{pro:estimation error}}
   Using the decomposition of the estimation error in \eqref{eq:J}, we combine the bounds on the terms \( J_{1} \), \( J_{2} \), \( J_{3} \) and \( J_{4} \) to conclude the proof.  The term \( J_{1} = \norm{ \LK^{\frac{1 - \gamma}{2}} \, \LKlam^{-1/2} } \) is bounded using \lemref{lem:cordes} and \lemref{lem:sup fraction}:
    \[
        J_{1}
        \le \norm{ \LK^{1 - \gamma} \, \LKlam^{-1} }^{1/2}
        \le \LE( \sup_{t \ge 0} \frac{t^{1 - \gamma}}{t + \lambda} \RI)^{1/2}
        \le \lambda^{-\gamma/2}.
    \]

    For \( J_{2} = \norm{ \LKlam^{1/2} \, \hLKlam^{-1/2} } \) in \eqref{eq:J}, if \neweqref{eq:R1}{R1} is satisfied and \( n \) is sufficiently large such that \eqref{eq:S1} holds, then
    \[
        4 \LE( L M_{\alpha}^{2} \cdot \frac{\lambda^{-\alpha}}{n} + \sigma M_{\alpha}^{1 + \frac{1}{p}} \cdot \frac{\lambda^{-\frac{1 + 1/p}{2} \alpha} \calN^{\frac{1 - 1/p}{2}}(\lambda)}{\sqrt{n}} \RI) \log \frac{6}{\delta}
        \le \frac{1}{4} + \frac{1}{4}
        = \frac{1}{2}.
    \]
    and by \proref{pro:LKlam^1/2 hLKlam^-1/2}, with probability at least \( 1 - \delta/3 \),
    \[
        J_{2} = \norm{ \LKlam^{1/2} \, \hLKlam^{-1/2} }\le \sqrt{2}.
    \]

    By \proref{pro:part 1}, under \eqref{eq:S1} and \neweqref{eq:R1}{R1},
    \[
        J_{3} = \norm{ \hLKlam^{1/2} \, \philam(\hLK) \, (\hSKa \, \bsy - \hLK \, \flam) }_{\calH}= O \LE( \lambda^{r} \log \frac{6}{\delta} \RI)
    \]
  holds with probability at least \( 1 - (2 \delta) / 3 \).

    The bound for \( J_{4} \), i.e., \( \norm{ \hLKlam^{1/2} \LE( I - \hLK \, \philam(\hLK) \RI) \flam }_{\calH} \), is given in \proref{pro:hLKlam^1/2 psilam(hLKlam) flam}: assuming \eqref{eq:S1},
    \begin{equation}
        \label{eq:rate 2}
        J_{4}
        \le 2 \sqrt{2} E F \norm{ \urho }_{\rhoTX} \cdot \LE(
        \lambda^{r}
        + 4 r \kappa^{2r-1} (L + \sigma) \cdot \lambda^{1/2} n^{-\frac{\min \family{ 2r, 3 } - 1}{4}} \log \frac{6}{\delta}
        \cdot \ind{r > 1}
        \RI)
    \end{equation}
    with probability at least \( 1 - 2\delta/3 \). For \( r > 1 \), the dominant term in \eqref{eq:rate 2} is
    \[
        \max \family{ \lambda^{r}, \lambda^{1/2} n^{-\frac{\min \family{ 2r, 3 } - 1}{4}} } \log \frac{6}{\delta}
        = \max \family{ \lambda^{r}, \lambda^{\frac{\min \family{ 2r, 3 } - 1}{4 s} + \frac{1}{2}} } \log \frac{6}{\delta}.
    \]
    Under \neweqref{eq:R2}{R2}, this term decays as \( O( \, \lambda^{r} \log (6/\delta) \, ) \).

   Substituting the bounds for \( J_{1} \), \( J_{2} \), \( J_{3} \) and \( J_{4} \) into decomposition \eqref{eq:J}, we conclude that if conditions \eqref{eq:S1}, \neweqref{eq:R1}{R1}, and \neweqref{eq:R2}{R2} are satisfied, then with probability at least \( 1 - \delta \), the estimation error satisfies
    \[
        \norm{ \hflam - \frho }_{\ranH{\gamma}} = O \LE( \lambda^{r - \frac{\gamma}{2}} \log \frac{6}{\delta} \RI),
    \]
    which completes the proof of \proref{pro:estimation error}.
\end{proofof}

Now we are ready to prove \thmref{thm:importance weighting}.
\begin{proofof}{\thmref{thm:importance weighting}}
    To apply \proref{pro:estimation error}, when \( 2r > \alpha_{0} \), we select the parameter \( s \) as
    \[
        s = \LE( 2r + \frac{1}{\beta} + \frac{\alpha_{0} + \epsilon - 1/\beta}{p} \RI)^{-1};
    \]
    whereas for \( 2r \le \alpha_{0} \), we choose
    \[
        s = \LE( \alpha_{0} + \epsilon + \frac{1}{\beta} + \frac{\alpha_{0} + \epsilon - 1/\beta}{p} \RI)^{-1}.
    \]
    In the case \( 2r > \alpha_{0} \), the parameter \( \epsilon \) must satisfy \( \epsilon \in (0, 2r - \alpha_{0}) \); otherwise, any \( \epsilon > 0 \) is permitted. This selection ensures that both conditions \neweqref{eq:R1}{R1} and \neweqref{eq:R2}{R2} are satisfied with \( \alpha = \alpha_{0} + \epsilon / 2 \). Furthermore, we assume that \( n \) is sufficiently large to satisfy \eqref{eq:n} as stated in \thmref{thm:importance weighting}, which guarantees that \eqref{eq:S1} holds. Consequently, \proref{pro:estimation error} applies with probability at least \( 1 - \delta \). Combining this result with \proref{pro:approximation error}, we obtain the error bound
    \[
        \norm{ \hflam - \frho }_{\ranH{\gamma}}
        = O \LE( \lambda^{r - \frac{\gamma}{2}} \log \frac{6}{\delta} \RI).
    \]
    Substituting \( \lambda = n^{-s} \) completes the proof.
\end{proofof}

\subsection{Proof of \thmref{thm:truncated weighting}}
\label{sec:proof 2}

The proof of \thmref{thm:truncated weighting} follows a structure analogous to that of \thmref{thm:importance weighting}. We begin by decomposing the excess error into two components:
\[
    \norm{ \hflamD - \frho }_{\ranH{\gamma}}
    \le \underbrace{
    \vphantom{\frac{1}{\frac{1}{2}}}
    \norm{ \hflamD - \flam }_{\ranH{\gamma}} }_{\text{estimation error}}
    + \underbrace{
    \vphantom{\frac{1}{\frac{1}{2}}}
    \norm{ \flam - \frho }_{\ranH{\gamma}} }_{\text{approximation error}}.
\]

The approximation error bound was established in \proref{pro:approximation error}, exhibiting a convergence rate of \( O(\lambda^{r - \frac{\gamma}{2}}) \). The estimation error is bounded in the following proposition:
\begin{proposition}
    \label{pro:(T) estimation error}
    Suppose that \aspref{asp:ratio} holds with \( p \in [1, \infty) \), \aspref{asp:source} holds with \( r \in (0, \tau] \), \aspref{asp:eigenvalue} holds with \( \beta > 1 \), and \aspref{asp:embedding} holds with \( \alpha_{0} \in [1/\beta, 1) \). Define the truncated density ratio \( \wD(x) = \min \family{ w(x), D } \) with \( D = n^{\nu} \); set \( \lambda = n^{-s} \), satisfying \neweqref{eq:R2}{R2} and the following conditions:
        \begin{enumerate}
            \item \( s \cdot \LE( 1 + \frac{1}{\beta} \RI) < \min \family{ p (m-1) \cdot \nu, \frac{1 - \nu}{\alpha} } \); \hfill \label{eq:(T) R3} \textup{(\textbf{R3})}\medskip

            \item \( s \cdot \max \family{ \alpha + \frac{1}{\beta}, 2r + \frac{1}{\beta} } \le 1 - \nu \); \hfill \label{eq:(T) R4} \textup{(\textbf{R4})}\medskip

            \item \( p (m-1) \cdot \nu \ge \frac{1}{2} \). \hfill \label{eq:(T) R5} \textup{(\textbf{R5})}
        \end{enumerate}
        Here, \( m \ge 2 \) is fixed; if \( 2r \le \alpha_{0} \), then \( \alpha \) can be chosen arbitrarily in \( (\alpha_{0}, 1] \); if \( 2r > \alpha_{0} \), then we require \( \alpha_{0} < \alpha \le \min \family{ 2r, 1 } \). Then, for any \( \delta \in (0, 1) \) and
    \begin{equation}
        \label{eq:(T) S2} \tag{\textbf{S2}}
        \begin{aligned}
            n \ge \max \biggg\{ & \LE( 2 \kappa C_{\calN}^{1/2} \LE( \frac{1}{2} m! L^{m-2} \sigma^{2} \RI)^{p/2} \RI)^{\frac{2}{p (m-1) \cdot \nu - \LE( 1 + \frac{1}{\beta} \RI) s}},                                                                          \\
                                & \LE( 32 M_{\alpha}^{2} \log \frac{6}{\delta} \RI)^{\frac{1}{1 - \nu - s \alpha}}, \LE( 16 \sqrt{2} M_{\alpha} C_{\calN}^{1/2} \log \frac{6}{\delta} \RI)^{\frac{2}{1 - \nu - \LE( 1 + \frac{1}{\beta} \RI) s \alpha}} \biggg\}
        \end{aligned}
    \end{equation}
    with \( M_{\alpha} = \norm{ \ranH{\alpha} \hookrightarrow L^{\infty}(\calX, \rhoTX) } \), the following convergence bound holds with probability at least \( 1 - \delta \):
    \[
        \norm{ \hflamD - \flam }_{\ranH{\gamma}}
        = O \LE( \lambda^{r - \frac{\gamma}{2}} \log \frac{6}{\delta} \RI),
        \quad 0 \le \gamma \le \min \family{ 2r, 1 }.
    \]
\end{proposition}
\begin{proof}
    To analyze the estimation error, we define the expected operator \( \LKD \) of \( \hLKD \):
    \[
        \LKD \colon \calH \to \calH,
        \quad f \mapsto \int_{\calX} f(x) \wD(x) \, K(\cdot, x) \, \dd \rhoSx.
    \]
    Denote \( \LKlamD = \LKD + \lambda I \) and \( \hLKlamD = \hLKD + \lambda I \). Following the same decomposition strategy as in \eqref{eq:J}, we express the estimation error as:
    \begin{equation}
        \label{eq:(T) J}
        \norm{ \hflamD - \flam }_{\ranH{\gamma}}
        \le J_{1}^{\dagger} \cdot J_{2}^{\dagger} \cdot (J_{3}^{\dagger} + J_{4}^{\dagger}),
    \end{equation}
    where
    \begin{alignat*}{3}
        J_{1}^{\dagger} & = \norm{ \LK^{\frac{1 - \gamma}{2}} \, (\LKlamD)^{-1/2} },
                        &                                                                                                   & \quad J_{2}^{\dagger} &  & = \norm{ (\LKlamD)^{1/2} \, (\hLKlamD)^{-1/2} },                                 \\
        J_{3}^{\dagger} & = \norm{ (\hLKlamD)^{1/2} \, \philam(\hLKD) \LE( (\hSKDa \, \bsy - \hLKD \, \flam \RI) }_{\calH},
                        &                                                                                                   & \quad J_{4}^{\dagger} &  & = \norm{ (\hLKlamD)^{1/2} \LE( I - \hLKD \, \philam(\hLKD) \RI) \flam }_{\calH}.
    \end{alignat*}

    The terms \( J_{1}^{\dagger} = \norm{ \LK^{\frac{1 - \gamma}{2}} \, (\LKlamD)^{-1/2} } \) and \( J_{2}^{\dagger} = \norm{ (\LKlamD)^{1/2} \, (\hLKlamD)^{-1/2} } \) in \eqref{eq:(T) J} are bounded by \proref{pro:(T) LKlamD^1/2 hLKlamD^-1/2}. Under \eqref{eq:(T) S2} and \neweqref{eq:(T) R3}{R3}, we have
    \[
        \kappa \lambda^{-1/2} \cdot \calN^{1/2}(\lambda) \cdot \LE( D^{-(m-1)} \frac{1}{2} m! L^{m-2} \sigma^{2} \RI)^{p/2}
        \le \frac{1}{2},
    \]
    and
    \[
        4 \LE( 2 M_{\alpha}^{2} \cdot \frac{D \lambda^{-\alpha}}{n} + \sqrt{2} M_{\alpha} \cdot \LE( \frac{D \lambda^{-\alpha} \calN(\lambda) }{n} \RI)^{1/2} \RI) \log \frac{6}{\delta}
        \le \frac{1}{4} + \frac{1}{4}
        = \frac{1}{2},
    \]
    which implies
    \[
        J_{1}^{\dagger} \le \sqrt{2} \lambda^{-\gamma/2},
        \quad J_{2}^{\dagger} \le \sqrt{2}
    \]
    with probability at least \( 1 - \delta/3 \).

    As established in Proposition~\ref{pro:(T) part 1}, under conditions \eqref{eq:(T) S2}, \neweqref{eq:(T) R3}{R3}, and \neweqref{eq:(T) R4}{R4}, the term \( J_{3}^{\dagger} = \norm{ (\hLKlamD)^{1/2} \, \philam(\hLKD) \LE( (\hSKDa \, \bsy - \hLKD \, \flam \RI) }_{\calH} \) converges at the rate \( O( \lambda^{r} \log (6/\delta) ) \) with probability at least \( 1 - (2 \delta)/3 \).

    For the term \( J_{4}^{\dagger} = \norm{ (\hLKlamD)^{1/2} \LE( I - \hLKD \, \philam(\hLKD) \RI) \flam }_{\calH} \), we apply \proref{pro:(T) hLKlamD^1/2 psilam(hLKlamD) flam}. Under \eqref{eq:(T) S2} and the condition \( D^{-(m-1)p} \le n^{-1/2} \), we have
    \begin{equation}
        \label{eq:(T) rate 2}
        \begin{aligned}
            J_{4}^{\dagger}
             & \le 4 E F \norm{ \urho }_{\rhoTX} \cdot \bigg( \lambda^{r} + 4 r \kappa^{2r - 1} \LE( L + \sigma + \LE( \frac{1}{2} m! L^{m-2} \sigma^{2} \RI)^{p} \RI)                                  \\
             & \eqspace \hphantom{4 E F \norm{ \urho }_{\rhoTX} \cdot \bigg( \bigg( \lambda^{r} +} \cdot \lambda^{1/2} n^{-\frac{\min \family{ 2r, 3 } - 1}{4}} \log \frac{6}{\delta} \cdot \ind{r > 1}
            \bigg)
        \end{aligned}
    \end{equation}
    with probability at least \( 1 - (2 \delta) / 3 \). For \( D = n^{\nu} \), the condition \( D^{-(m-1)p} \le n^{-1/2} \) in \proref{pro:(T) hLKlamD^1/2 psilam(hLKlamD) flam} requires \neweqref{eq:(T) R5}{R5}. When \( r > 1 \), the rate in \eqref{eq:(T) rate 2} coincides with that of \eqref{eq:rate 2}. Hence, under the additional assumption \neweqref{eq:R2}{R2}, this term also decays at the rate \( O( \, \lambda^{r} \log (6/\delta) \, ) \).

    Combining these results, we conclude that if conditions \eqref{eq:(T) S2}, \neweqref{eq:R2}{R2}, \neweqref{eq:(T) R3}{R3}, \neweqref{eq:(T) R4}{R4}, and \neweqref{eq:(T) R5}{R5} are all satisfied, then these norm bounds hold simultaneously with probability at least \( 1 - \delta \). Therefore, \( \norm{ \hflamD - \frho }_{\ranH{\gamma}} \) decays at the rate \( O( \, \lambda^{r - \frac{\gamma}{2}} \log (6/\delta) \, ) \), which completes the proof of \proref{pro:(T) estimation error}.
\end{proof}

The bounds for \( J_{1}^{\dagger} = \norm{ \LK^{\frac{1 - \gamma}{2}} \, (\LKlamD)^{-1/2} } \) and \( J_{2}^{\dagger} = \norm{ (\LKlamD)^{1/2} \, (\hLKlamD)^{-1/2} } \) in \eqref{eq:(T) J} are established in the following proposition:
\begin{proposition}
    \label{pro:(T) LKlamD^1/2 hLKlamD^-1/2}
    Suppose that \aspref{asp:ratio} holds with \( p \in [1, \infty) \), and that \( \calH \) has embedding index \( \alpha_{0} < 1 \). For any \( \delta \in (0, 1) \) and \( \alpha \in (\alpha_{0}, 1] \), if \( n \), \( \lambda \), and \( D \) satisfy:
    \[
        \kappa \lambda^{-1/2} \cdot \calN^{1/2}(\lambda) \cdot \LE( D^{-(m'-1)} \frac{1}{2} m'! L^{m'-2} \sigma^{2} \RI)^{p/2}
        \le \frac{1}{2},
        \quad \exists m' \ge 2,
    \]
    and
    \[
        4 \LE( \frac{\tilde{L}_{3}}{n} + \frac{\tilde{\sigma}_{3}}{\sqrt{n}} \RI) \log \frac{6}{\delta}
        \le \frac{1}{2},
    \]
    where the parameters are defined as:
    \[
        \tilde{L}_{3}
        = 2 M_{\alpha}^{2} \cdot D \lambda^{-\alpha},
        \quad \tilde{\sigma}_{3}
        = \LE( 2 M_{\alpha}^{2} \cdot D \lambda^{-\alpha} \calN(\lambda) \RI)^{1/2},
        \quad M_{\alpha} = \norm{ \ranH{\alpha} \hookrightarrow L^{\infty}(\calX, \rhoTX) },
    \]
    then
    \[
        J_{1}^{\dagger} = \norm{ \LK^{\frac{1 - \gamma}{2}} \, (\LKlamD)^{-1/2} }
        \le \sqrt{2} \lambda^{-\gamma/2},
    \]
    and with probability at least \( 1 - \delta/3 \),
    \[
        J_{2}^{\dagger} = \norm{ (\LKlamD)^{1/2} \, (\hLKlamD)^{-1/2} } \le \sqrt{2}.
    \]
\end{proposition}
\begin{proof}
    We first bound the operator norm \( J_{2}^{\dagger} = \norm{ (\LKlamD)^{1/2} \, (\hLKlamD)^{-1/2} } \). Define the random operator \( \xi(x) = (\LKlamD)^{-1/2} \circ (\wD(x) \, \Kx \Kxa) \circ (\LKlamD)^{-1/2} \), where \( \Kx \) and \( \Kxa \) are defined in \eqref{eq:Kx}. Following the methodology of \lemref{lem:LKlam^-1/2 (LK - hLK) LKlam^-1/2}, we estimate the moments \( \E{ \norm{ \xi }_{\HS}^{m} } \). Note that the Hilbert-Schmidt norm \( \norm{ \cdot }_{\HS} \) satisfies:
    \begin{align*}
        \norm{ (\LKlamD)^{-1/2} \circ (\Kx \Kxa) \circ (\LKlamD)^{-1/2} }_{\HS}
         & = \norm{ (\LKlamD)^{-1/2} \, K(\cdot, x) }_{\calH}^{2}                                                     \\
         & \le \norm{ (\LKlamD)^{-1/2} \, \LKlam^{1/2} }^{2} \cdot \norm{ \LKlam^{-1/2} \, K(\cdot, x) }_{\calH}^{2}.
    \end{align*}
    The inverse \( (\LKlamD)^{-1} \) can be expanded as:
    \begin{equation}
        \label{eq:LKlamD-1}
        \begin{aligned}
            (\LKlamD)^{-1}
             & = (\LKD + \lambda I)^{-1}
            = (\LKD - \LK + \LK + \lambda I)^{-1}                             \\
             & = \LE( \LKlam - (\LK - \LKD) \RI)^{-1}
            = \LE( \LE( I - (\LK - \LKD) \, \LKlam^{-1} \RI) \LKlam \RI)^{-1} \\
             & = \LKlam^{-1} \LE( I - (\LK - \LKD) \, \LKlam^{-1} \RI)^{-1}.
        \end{aligned}
    \end{equation}
    By \lemref{lem:cordes}, we have
    \begin{align*}
        \norm{ (\LKlamD)^{-1/2} \, \LKlam^{1/2} }^{2}
         & \le \norm{ (\LKlamD)^{-1} \, \LKlam }
        = \norm{ \LKlam \, (\LKlamD)^{-1} }                           \\
         & = \norm{ \LE( I - (\LK - \LKD) \, \LKlam^{-1} \RI)^{-1} }.
    \end{align*}
    Under the given conditions, \lemref{lem:(T) (LK - LKD) LKlam^-1} implies \( \norm{ (\LK - \LKD) \, \LKlam^{-1} } \le 1/2 \), so that
    \begin{equation}
        \label{eq:(I - (LK - LKD) LKlam-1)-1}
        \norm{ \LE( I - (\LK - \LKD) \, \LKlam^{-1} \RI)^{-1} }
        \le \sum_{k=0}^{\infty} \norm{ (\LK - \LKD) \, \LKlam^{-1} }^{k}
        \le 2.
    \end{equation}
    Combining this with \lemref{lem:LKlam^-1/2 Kx} yields the uniform bound:
    \begin{equation}
        \label{eq:LKlamD-1/2 Kx}
        \begin{aligned}
            \norm{ (\LKlamD)^{-1/2} \, K(\cdot, x) }_{\calH}^{2}
             & \le \norm{ (\LKlamD)^{-1/2} \, \LKlam^{1/2} }^{2} \cdot \norm{ \LKlam^{-1/2} \, K(\cdot, x) }_{\calH}^{2} \\
             & \le 2 M_{\alpha}^{2} \lambda^{-\alpha},
            \quad \rhoTX \text{-a.e. } x \in \calX.
        \end{aligned}
    \end{equation}

    Furthermore, the identity
    \[
        \int_{\calX} \norm{ (\LKlamD)^{-1/2} \, K(\cdot, x) }_{\calH}^{2} \wD(x) \, \dd \rhoSx
        = \Tr{ (\LKlamD)^{-1} \, \LKD }
    \]
    holds. Since \( \LK - \LKD \) is positive semi-definite and \( x \mapsto x (x + \lambda)^{-1} \) is operator monotone,
    \begin{equation}
        \label{eq:Tr(LKlamD-1 LKD)}
        \Tr{ (\LKlamD)^{-1} \, \LKD }
        \le \Tr{ (\LKlam)^{-1} \, \LK }
        =\calN(\lambda).
    \end{equation}

    These estimates imply:
    \begin{align*}
        \E{ \norm{ \xi }_{\HS}^{m} }
         & = \int_{\calX} \norm{ (\LKlamD)^{-1/2} \circ (\Kx \Kxa) \circ (\LKlamD)^{-1/2} }_{\HS}^{m} ( \, \wD(x) \, )^{m} \, \dd \rhoSx                             \\
         & = \int_{\calX} \norm{ (\LKlamD)^{-1/2} \, K(\cdot, x) }_{\calH}^{2m} ( \, \wD(x) \, )^{m} \, \dd \rhoSx                                                   \\
         & \le (2 M_{\alpha}^{2} \lambda^{-\alpha})^{m-1} \cdot D^{m-1} \cdot \int_{\calX} \norm{ (\LKlamD)^{-1/2} \, K(\cdot, x) }_{\calH}^{2} \wD(x) \, \dd \rhoSx \\
         & \le (2 M_{\alpha}^{2} \lambda^{-\alpha})^{m-1} \cdot D^{m-1} \cdot \calN(\lambda).
    \end{align*}
    Hence, \( \E{ \norm{ \xi }_{\HS}^{m} } \le \frac{1}{2} m! \tilde{L}_{3}^{m-2} \tilde{\sigma}_{3}^{2} \), where
    \[
        \tilde{L}_{3}
        = 2 M_{\alpha}^{2} \cdot D \lambda^{-\alpha},
        \quad \tilde{\sigma}_{3}
        = \LE( 2 M_{\alpha}^{2} \cdot D \lambda^{-\alpha} \calN(\lambda) \RI)^{1/2}.
    \]
    Applying \lemref{lem:bernstein} under the stated conditions, we obtain
    \[
        \norm{ (\LKlamD)^{-1/2} \, (\LKD - \hLKD) \, (\LKlamD)^{-1/2} } \le \frac{1}{2}
    \]
    with probability at least \( 1 - \delta / 3 \). Then, as in \proref{pro:LKlam^1/2 hLKlam^-1/2}, it follows that
    \[
        J_{2}^{\dagger} = \norm{ (\LKlamD)^{1/2} \, (\hLKlamD)^{-1/2} } \le \sqrt{2}.
    \]

    To complete the proof, we bound \( J_{1}^{\dagger} = \norm{ \LK^{\frac{1 - \gamma}{2}} \, (\LKlamD)^{-1/2} } \). Applying \lemref{lem:cordes} and \lemref{lem:sup fraction} yields
    \begin{align*}
        J_{1}^{\dagger}
         & \le \norm{ \LK^{\frac{1 - \gamma}{2}} \, \LKlam^{-1/2} }
        \cdot \norm{ \LKlam^{1/2} \, (\LKlamD)^{-1/2} }                     \\
         & \le \norm{ \LK^{1 - \gamma} \, \LKlam^{-1} }^{1/2}
        \cdot \norm{ \LE( I - (\LK - \LKD) \, \LKlam^{-1} \RI)^{-1} }^{1/2} \\
         & \le \sqrt{2} \lambda^{-\gamma/2}. \qedhere
    \end{align*}
\end{proof}

The bound for the term \( J_{3}^{\dagger} = \norm{ (\hLKlamD)^{1/2} \, \philam(\hLKD) \LE( (\hSKDa \, \bsy - \hLKD \, \flam \RI) }_{\calH} \) in \eqref{eq:(T) J} is established in the following proposition:
\begin{proposition}
    \label{pro:(T) part 1}
    Suppose that \aspref{asp:ratio} holds with \( p \in [1, \infty) \), \aspref{asp:source} holds with \( r \in (0, \tau] \), \aspref{asp:eigenvalue} holds with \( \beta > 1 \), and \aspref{asp:embedding} holds with \( \alpha_{0} \in [1/\beta, 1) \). Let \( \lambda = n^{-s} \) with \( s \) satisfying \neweqref{eq:(T) R3}{R3} and \neweqref{eq:(T) R4}{R4}, where the parameter \( \alpha \) is chosen as follows: if \( 2r \le \alpha_{0} \), then \( \alpha \in (\alpha_{0}, 1] \); if \( 2r > \alpha_{0} \), then \( \alpha_{0} < \alpha \le \min \family{ 2r, 1 } \). Then, for any \( \delta \in (0, 1) \) and for all sufficiently large \( n \) satisfying \eqref{eq:(T) S2},
    \[
        J_{3}^{\dagger} = \norm{ (\hLKlamD)^{1/2} \, \philam(\hLKD) \LE( (\hSKDa \, \bsy - \hLKD \, \flam \RI) }_{\calH} = O \LE( \lambda^{r} \log \frac{6}{\delta} \RI)
    \]
    holds with probability at least \( 1 - (2 \delta) / 3 \).
\end{proposition}
\begin{proof}
    We begin with the decomposition:
    \begin{align*}
        J_{3}^{\dagger} & = \norm{ (\hLKlamD)^{1/2} \, \philam(\hLKD) \LE( (\hSKDa \, \bsy - \hLKD \, \flam \RI) }_{\calH}           \\
                        & \le \norm{ (\hLKlamD)^{1/2} \, \philam(\hLKD) \, (\hLKlamD)^{1/2} }
        \cdot \norm{ (\hLKlamD)^{-1/2} \, (\LKlamD)^{1/2} }
        \cdot \Big\| (\LKlamD)^{-1/2}                                                                                                \\
                        & \eqspace \hphantom{\Big\|} \circ \LE( (\hSKDa \, \bsy - \hLKD \, \flam \RI) \Big\|_{\calH}                 \\
                        & \le 2 E \cdot \sqrt{2} \cdot \norm{ (\LKlamD)^{-1/2} \LE( (\hSKDa \, \bsy - \hLKD \, \flam \RI) }_{\calH},
    \end{align*}
    where the last inequality follows from the filter function property \eqref{eq:philam} (with \( \theta = 0 \) and \( \theta = 1 \)) and \proref{pro:(T) LKlamD^1/2 hLKlamD^-1/2} (which holds under \eqref{eq:(T) S2} and \neweqref{eq:(T) R3}{R3} with probability at least \( 1 - \delta/3 \)). We further decompose the remaining term:
    \begin{equation}
        \label{eq:(T) decomp J3}
        \begin{aligned}
             & \eqspace \norm{ (\LKlamD)^{-1/2} \LE( (\hSKDa \, \bsy - \hLKD \, \flam \RI) }_{\calH}                                        \\
             & \le \norm{ (\LKlamD)^{-1/2} \LE( \LE( (\hSKDa \, \bsy - \hLKD \, \flam \RI) - (\LKD \, \frho - \LKD \, \flam) \RI) }_{\calH}
            + \Big\| (\LKlamD)^{-1/2}                                                                                                       \\
             & \eqspace \hphantom{\Big\|} \circ (\LKD \, \frho - \LKD \, \flam) \Big\|_{\calH}.
        \end{aligned}
    \end{equation}
    To bound the first component in \eqref{eq:(T) decomp J3}, define the random variable
    \begin{align*}
        \xi
         & = \xi(z)
        = (\LKlamD)^{-1/2} \, \wD(x) \, (\Kx \, y - \Kx \Kxa \, \flam)
        = (\LKlamD)^{-1/2} \, \wD(x) \, \Kx \, ( \, y - \flam(x) )           \\
         & = (\LKlamD)^{-1/2} \, K(\cdot, x) \cdot \wD(x) \, (y - \flam(x)).
    \end{align*}
    From \eqref{eq:y - flam(x)}, we obtain the uniform bound
    \[
        |y - \flam(x)|
        \le M_{\alpha} (E + F) \norm{ \urho }_{\rhoTX} \cdot \lambda^{-\LE( \frac{\alpha}{2} - r \RI)} + 2 G,
        \quad \rhoTX \text{-a.e. } x \in \calX.
    \]
    Moreover, as shown in the proof of \proref{pro:(T) LKlamD^1/2 hLKlamD^-1/2} (see \eqref{eq:LKlamD-1/2 Kx}),
    \[
        \norm{ (\LKlamD)^{-1/2} \, K(\cdot, x) }_{\calH}
        \le \sqrt{2} M_{\alpha} \lambda^{-\alpha/2},
        \quad \rhoTX \text{-a.e. } x \in \calX,
    \]
    and the integral bound (see \eqref{eq:Tr(LKlamD-1 LKD)}):
    \[
        \int_{\calX} \norm{ (\LKlamD)^{-1/2} \, K(\cdot, x) }_{\calH}^{2} \wD(x) \, \dd \rhoSx
        \le \calN(\lambda).
    \]
    Consequently, for \( m \ge 2 \),
    \[
        \int_{\calX} \norm{ (\LKlamD)^{-1/2} \, K(\cdot, x) }_{\calH}^{m} \wD(x) \, \dd \rhoSx
        \le \LE( \sqrt{2} M_{\alpha} \lambda^{-\alpha/2} \RI)^{m - 2} \calN(\lambda).
    \]
    Combining these estimates yields
    \begin{align*}
         & \eqspace \E{ \norm{ \xi }_{\calH}^{m} }                                                                                      \\
         & = \int_{\calX} \norm{ (\LKlamD)^{-1/2} \, K(\cdot, x) }_{\calH}^{m} ( \, \wD(x) \, )^{m} \, |y - \flam(x)|^{m} \, \dd \rhoSx \\
         & \le \LE( M_{\alpha} (E + F) \norm{ \urho }_{\rhoTX} \cdot \lambda^{-\LE( \frac{\alpha}{2} - r \RI)} + 2 G \RI)^{m}
        \cdot D^{m-1}
        \cdot \int_{\calX} \norm{ (\LKlamD)^{-1/2} \, K(\cdot, x) }_{\calH}^{m} \wD(x) \, \dd \rhoSx                                    \\
         & \le \LE( M_{\alpha} (E + F) \norm{ \urho }_{\rhoTX} \cdot \lambda^{-\LE( \frac{\alpha}{2} - r \RI)} + 2 G \RI)^{m}
        \cdot D^{m-1}
        \cdot \LE( \sqrt{2} M_{\alpha} \lambda^{-\alpha/2} \RI)^{m - 2} \calN(\lambda).
    \end{align*}
    After simplification, we obtain the moment bound
    \[
        \E{ \norm{ \xi }_{\calH}^{m} }
        \le \frac{1}{2} m! \tilde{L}_{4}^{m-2} \tilde{\sigma}_{4}^{2}
    \]
    with parameters
    \begin{align*}
        \tilde{L}_{4}
         & = \sqrt{2} M_{\alpha} \cdot \LE( M_{\alpha} (E + F) \norm{ \urho }_{\rhoTX} \cdot \lambda^{-\LE( \frac{\alpha}{2} - r \RI)} + 2 G \RI)
        \cdot D \lambda^{-\alpha/2},                                                                                                              \\
        \tilde{\sigma}_{4}
         & = \LE( M_{\alpha} (E + F) \norm{ \urho }_{\rhoTX} \cdot \lambda^{-\LE( \frac{\alpha}{2} - r \RI)} + 2 G \RI)
        \cdot D^{1/2} \calN^{\frac{1}{2}}(\lambda).
    \end{align*}
    Applying \lemref{lem:bernstein}, we conclude that
    \[
        \norm{ (\LKlamD)^{-1/2} \LE( \LE( (\hSKDa \, \bsy - \hLKD \, \flam \RI) - (\LKD \, \frho - \LKD \, \flam) \RI) }_{\calH}
        \le 4 \LE( \frac{\tilde{L}_{4}}{n} + \frac{\tilde{\sigma}_{4}}{\sqrt{n}} \RI) \log \frac{6}{\delta}
    \]
    holds with probability at least \( 1 - \delta/3 \). Substituting the expressions for \( \tilde{L}_{4} \) and \( \tilde{\sigma}_{4} \), and using \( \lambda = n^{-s} \), we obtain the asymptotic rate:
    \begin{align*}
         & \eqspace \LE( \frac{(\lambda^{r - \alpha} + \lambda^{-\alpha/2}) D}{n}
        + \frac{\LE( \lambda^{-\LE( \frac{\alpha}{2} - r \RI)} + 1 \RI) D^{1/2} \calN^{1/2}(\lambda)}{\sqrt{n}} \RI) \log \frac{6}{\delta} \\
         & \asymp \LE(
        \frac{n^{s \alpha + \nu} + n^{s \LE( \frac{\alpha}{2} + r \RI) + \nu}}{n}
        + \LE( \frac{n^{s \LE( \alpha + \frac{1}{\beta} \RI) + \nu} + n^{s \LE( 2 r + \frac{1}{\beta} \RI) + \nu}}{n} \RI)^{1/2}
        \RI) \lambda^{r} \log \frac{6}{\delta}.
    \end{align*}
    Condition \neweqref{eq:(T) R4}{R4} ensures that this expression is \( O( \, \lambda^{r} \log (6/\delta) \, ) \).

    For the second component \( \norm{ (\LKlamD)^{-1/2} \, (\LKD \, \frho - \LKD \, \flam) }_{\calH} \) in \eqref{eq:(T) decomp J3}, note that \( \LK - \LKD \) is positive semi-definite on \( \calH \), implying
    \[
        \norm{ (\LKD)^{1/2} \, f }_{\calH}^{2}
        = \inner{ \LKD \, f, f }_{\calH}
        \le \inner{ \LK \, f, f }_{\calH}
        = \norm{ \LK^{1/2} \, f }_{\calH}^{2}
        = \norm{ f }_{\rhoTX}^{2},
        \quad \forall f \in \calH.
    \]
    Therefore,
    \begin{equation}
        \begin{aligned}
             & \eqspace \norm{ (\LKlamD)^{-1/2} \, (\LKD \, \frho - \LKD \, \flam) }_{\calH}                        \\
             & \le \norm{ (\LKlamD)^{-1/2} \, (\LKD)^{1/2} } \cdot \norm{ (\LKD)^{1/2} \, (\frho - \flam) }_{\calH}
            \le \norm{ (\LKD)^{1/2} \, (\frho - \flam) }_{\calH}                                                    \\
             & \le \norm{ \frho - \flam }_{\rhoTX}
            \le F \norm{ \urho }_{\rhoTX} \cdot \lambda^{r},
        \end{aligned}
    \end{equation}
    where the last inequality follows from \proref{pro:approximation error} with \( \gamma = 0 \), confirming an \( O(\lambda^{r}) \) rate.

    Combining both bounds and accounting for the probabilistic estimates, we conclude that under \eqref{eq:(T) S2}, \neweqref{eq:(T) R3}{R3}, and \neweqref{eq:(T) R4}{R4}, the target term \( J_{3}^{\dagger} \) is bounded by \( O( \, \lambda^{r} \log (6/\delta) \, ) \) with probability at least \( 1 - (2\delta)/3 \).
\end{proof}

Finally, the following proposition bounds the term \( J_{4}^{\dagger} = \norm{ (\hLKlamD)^{1/2} \LE( I - \hLKD \, \philam(\hLKD) \RI) \flam }_{\calH} \) in \eqref{eq:(T) J}:
\begin{proposition}
    \label{pro:(T) hLKlamD^1/2 psilam(hLKlamD) flam}
    Suppose that \aspref{asp:source} holds with \( r \in (0, \tau] \), and assume the conditions of \proref{pro:(T) LKlamD^1/2 hLKlamD^-1/2}. If \( D^{-(m-1)p} \le n^{-1/2} \), then for any \( r \in (0, \tau] \) and \( \delta \in (0, 1) \), with probability at least \( 1 - (2 \delta) / 3 \), we have:
    \begin{align*}
        J_{4}^{\dagger}
         & = \norm{ (\hLKlamD)^{1/2} \LE( I - \hLKD \, \philam(\hLKD) \RI) \flam }_{\calH} \\
         & \le 4 E F \norm{ \urho }_{\rhoTX} \cdot \LE(
        \lambda^{r}
        + \Delta^{\dagger} \cdot \lambda^{1/2} n^{-\frac{\min \family{ 2r, 3 } - 1}{4}} \log \frac{6}{\delta} \cdot \ind{r > 1}
        \RI),
    \end{align*}
    where
    \[
        \Delta^{\dagger} = 4 r \kappa^{2r - 1} \LE( L + \sigma + \LE( \frac{1}{2} m! L^{m-2} \sigma^{2} \RI)^{p} \RI)
    \]
    is a constant independent of \( n \) and \( \delta \).
\end{proposition}
\begin{proof}
    We extend the approach from \proref{pro:hLKlam^1/2 psilam(hLKlam) flam} through a case analysis based on the source condition exponent \( r \):
    \begin{itemize}
        \item \( 0 < r < 1/2 \): Starting from the expansion:
              \[
                  \begin{aligned}
                      J_{4}^{\dagger} & = \norm{ (\hLKlamD)^{1/2} \LE( I - \hLKD \, \philam(\hLKD) \RI) \flam }_{\calH}                                                \\
                                      & \le \norm{ (\hLKlamD)^{1/2} \LE( I - \hLKD \, \philam(\hLKD) \RI) } \cdot \norm{ \philam(\LK) \, \LK^{r+1} \, \urho }_{\calH}.
                  \end{aligned}
              \]
              Following the proof of \proref{pro:hLKlam^1/2 psilam(hLKlam) flam}, we derive the bounds:
              \[
                  \norm{ (\hLKlamD)^{1/2} \LE( I - \hLKD \, \philam(\hLKD) \RI) }
                  \le 2 F \lambda^{1/2},
              \]
              and
              \[
                  \norm{ \philam(\LK) \, \LK^{r+1} \, \urho }_{\calH}
                  \le E \norm{ \urho }_{\rhoTX} \cdot \lambda^{r - \frac{1}{2}},
              \]
              which together yield
              \begin{equation}
                  \label{eq:(T) case1}
                  J_{4}^{\dagger}
                  \le 2 E F \norm{ \urho }_{\rhoTX} \cdot \lambda^{r}.
              \end{equation}

        \item \( 1/2 \le r \le 1 \): We proceed with the decomposition:
              \begin{align*}
                  J_{4}^{\dagger} & = \norm{ (\hLKlamD)^{1/2} \LE( I - \hLKD \, \philam(\hLKD) \RI) \flam }_{\calH}                                                          \\
                                  & \le \norm{ (\hLKlamD)^{1/2} \LE( I - \hLKD \, \philam(\hLKD) \RI) \philam(\LK) \, \LK^{r + \frac{1}{2}} } \cdot \norm{ \urho }_{\rhoTX}.
              \end{align*}
              Since \( r - 1/2 \ge 0 \), we have
              \begin{align*}
                   & \eqspace \norm{ (\hLKlamD)^{1/2} \LE( I - \hLKD \, \philam(\hLKD) \RI) \philam(\LK) \, \LK^{r + \frac{1}{2}} }                        \\
                   & = \bigg\|
                  (\hLKlamD)^{1/2} \LE( I - \hLKD \, \philam(\hLKD) \RI) (\hLKlamD)^{r - \frac{1}{2}}
                  \circ (\hLKlamD)^{-\LE( r - \frac{1}{2} \RI)} \, (\LKlamD)^{r - \frac{1}{2}}                                                             \\
                   & \eqspace \hphantom{\bigg\|} \circ (\LKlamD)^{-\LE( r - \frac{1}{2} \RI)} \, \LK^{r - \frac{1}{2}}
                  \circ \philam(\LK) \, \LK
                  \bigg\|                                                                                                                                  \\
                   & \le \norm{ (\hLKlamD)^{r} \LE( I - \hLKD \, \philam(\hLKD) \RI) }
                  \cdot \norm{ (\hLKlamD)^{-\LE( r - \frac{1}{2} \RI)} \, (\LKlamD)^{r - \frac{1}{2}} }
                  \cdot \norm{ (\LKlamD)^{-\LE( r - \frac{1}{2} \RI)} \, \LK^{r - \frac{1}{2}} }                                                           \\
                   & \eqspace \cdot \norm{ \philam(\LK) \, \LK }                                                                                           \\
                   & \le 2 F \lambda^{r} \cdot (\sqrt{2})^{2r - 1} \cdot \norm{ (\LKlamD)^{-\LE( r - \frac{1}{2} \RI)} \, \LK^{r - \frac{1}{2}} } \cdot E.
              \end{align*}
              In the last line, the bound for \( \norm{ (\hLKlamD)^{-\LE( r - \frac{1}{2} \RI)} \, (\LKlamD)^{r - \frac{1}{2}} } \) follows from \lemref{lem:cordes} and \proref{pro:(T) LKlamD^1/2 hLKlamD^-1/2} (which holds with probability at least \( 1 - \delta/3 \)). Moreover, the proof of \proref{pro:(T) LKlamD^1/2 hLKlamD^-1/2} establishes that (see \eqref{eq:LKlamD-1} and \eqref{eq:(I - (LK - LKD) LKlam-1)-1})
              \[
                  (\LKlamD)^{-1}
                  = \LKlam^{-1} \LE( I - (\LK - \LKD) \, \LKlam^{-1} \RI)^{-1},
              \]
              with
              \[
                  \norm{ \LE( I - (\LK - \LKD) \, \LKlam^{-1} \RI)^{-1} }
                  \le 2.
              \]
              Combining this with \lemref{lem:cordes} yields
              \begin{align*}
                  \norm{ \LK^{r - \frac{1}{2}} \, (\LKlamD)^{-\LE( r - \frac{1}{2} \RI)} }
                   & \le \norm{ \LK \, (\LKlamD)^{-1} }^{r - \frac{1}{2}}                                                                      \\
                   & \le \LE( \norm{ \LK \, \LKlam^{-1} } \cdot \norm{ \LE( I - (\LK - \LKD) \, \LKlam^{-1} \RI)^{-1} } \RI)^{r - \frac{1}{2}} \\
                   & \le 2^{r - \frac{1}{2}}.
              \end{align*}

              This leads to the final bound:
              \begin{equation}
                  \label{eq:(T) case2}
                  J_{4}^{\dagger}
                  \le 2^{2r} E F \norm{ \urho }_{\rhoTX} \cdot \lambda^{r}
                  \le 4 E F \norm{ \urho }_{\rhoTX} \cdot \lambda^{r}.
              \end{equation}

        \item \( r > 1 \): To estimate \( \norm{ (\hLKlamD)^{1/2} \LE( I - \hLKD \, \philam(\hLKD) \RI) \philam(\LK) \, \LK^{r + \frac{1}{2}} } \), we note that
              \[
                  \norm{ (\hLKlamD)^{1/2} \LE( I - \hLKD \, \philam(\hLKD) \RI) \philam(\LK) \, \LK^{r + \frac{1}{2}} }
                  \le \norm{ (\hLKlamD)^{1/2} \LE( I - \hLKD \, \philam(\hLKD) \RI) \LK^{r - \frac{1}{2}} } \cdot E.
              \]
              We then employ the decomposition:
              \begin{align*}
                  \norm{ (\hLKlamD)^{1/2} \LE( I - \hLKD \, \philam(\hLKD) \RI) \LK^{r - \frac{1}{2}} }
                   & \le \norm{ (\hLKlamD)^{1/2} \LE( I - \hLKD \, \philam(\hLKD) \RI) } \cdot \norm{ \LK^{r - \frac{1}{2}} - (\hLKD)^{r - \frac{1}{2}} } \\
                   & \eqspace + \norm{ (\hLKlamD)^{1/2} \LE( I - \hLKD \, \philam(\hLKD) \RI) (\hLKD)^{r - \frac{1}{2}} }                                 \\
                   & \le 2 F \lambda^{1/2} \cdot \norm{ \LK^{r - \frac{1}{2}} - (\hLKD)^{r - \frac{1}{2}} } + 2 F \lambda^{r},
              \end{align*}
              where the second inequality uses the identity \( \LK^{r - \frac{1}{2}} = \LE( \LK^{r - \frac{1}{2}} - (\hLKD)^{r - \frac{1}{2}} \RI) + (\hLKD)^{r - \frac{1}{2}} \). Applying \lemref{lem:A^s - B^s} gives:
              \begin{align*}
                  \norm{ \LK^{r - \frac{1}{2}} - (\hLKD)^{r - \frac{1}{2}} } \le
                  \begin{cases}
                      \norm{ \LK - \hLKD }^{r - \frac{1}{2}},                       & r \in (1, 3/2];
                      \medskip                                                                        \\
                      \LE( r - \frac{1}{2} \RI) \kappa^{2r-3} \norm{ \LK - \hLKD }, & r > 3/2.
                  \end{cases}
              \end{align*}
              According to \lemref{lem:(T) LK - hLKD}, with probability at least \( 1 - \delta / 3 \):
              \begin{align*}
                  \norm{ \LK - \hLKD }
                   & \le \kappa^{2} \LE( D^{-(m-1)} \frac{1}{2} m! L^{m-2} \sigma^{2} \RI)^{p} + 4 \kappa^{2} \LE( \frac{L}{n} + \frac{\sigma}{\sqrt{n}} \RI) \log \frac{6}{\delta} \\
                   & \le 4 \kappa^{2} \LE( L + \sigma + \LE( \frac{1}{2} m! L^{m-2} \sigma^{2} \RI)^{p} \RI) n^{-1/2} \log \frac{6}{\delta},
              \end{align*}
              where we use the assumption \( D^{-(m-1)p} \le n^{-1/2} \). Combining these results yields
              \begin{equation}
                  \label{eq:(T) case3}
                  J_{4}^{\dagger}
                  \le 2 E F \norm{ \urho }_{\rhoTX} \cdot \LE(
                  \lambda^{r}
                  + \Delta^{\dagger} \cdot \lambda^{1/2} n^{-\frac{\min \family{ 2r, 3 } - 1}{4}} \log \frac{6}{\delta}
                  \RI)
              \end{equation}
              with
              \[
                  \Delta^{\dagger} = 4 r \kappa^{2r - 1} \LE( L + \sigma + \LE( \frac{1}{2} m! L^{m-2} \sigma^{2} \RI)^{p} \RI).
              \]
    \end{itemize}

    The proof is completed by combining the bounds from \eqref{eq:(T) case1}, \eqref{eq:(T) case2}, and \eqref{eq:(T) case3}.
\end{proof}

Now we are ready to prove \thmref{thm:truncated weighting}.
\begin{proofof}{\thmref{thm:truncated weighting}}
    To apply \proref{pro:(T) estimation error}, for fixed \( m \ge 2 \), we define:
    \[
        \nu = \frac{1}{p(m - 1) + 1},
    \]
    and select the regularization parameter as:
    \[
        s =
        \begin{cases}
            \frac{1 - \nu}{2r + 1/\beta},           & 2r > 1;
            \smallskip                                          \\
            \frac{1 - \nu}{1 + \epsilon + 1/\beta}, & 2r \le 1,
        \end{cases}
    \]
    where \( \epsilon > 0 \) is an arbitrarily small constant. These choices ensure that conditions \neweqref{eq:R2}{R2}, \neweqref{eq:(T) R3}{R3}, \neweqref{eq:(T) R4}{R4}, and \neweqref{eq:(T) R5}{R5} are simultaneously satisfied with \( \alpha = 1 \); note that our selection is independent of \( \alpha_{0} \). Furthermore, we choose \( n \) sufficiently large so that condition \eqref{eq:(T) S2} holds, as specified in \eqref{eq:(T) n} of \thmref{thm:truncated weighting}. In particular, by \lemref{lem:embedding norm} and the assumption \( \sup_{x \in \calX} K(x, x) \leq \kappa^{2} \), we may replace \( M_{1} \) with \( \kappa \) in \eqref{eq:(T) n}. Consequently, \proref{pro:(T) estimation error} holds with probability at least \( 1 - \delta \). Combining this result with \proref{pro:approximation error}, we obtain
    \[
        \norm{ \hflamD - \frho }_{\ranH{\gamma}}
        = O \LE( \lambda^{r - \frac{\gamma}{2}} \log \frac{6}{\delta} \RI).
    \]
    Substituting \( \lambda = n^{-s} \) completes the proof of \thmref{thm:truncated weighting}.
\end{proofof}

\newpage
\printbibliography[heading=bibintoc, title=\ebibname]

\newpage
\appendix
\section{Appendix}
\label{sec:appendix}

This appendix presents auxiliary lemmas referenced in \secref{sec:proof}. Throughout this appendix, unless explicitly stated otherwise, all expectations and probabilities are computed with respect to \( x \sim \rhoSX \).

We first introduce the following lemma, which establishes bounds for the effective dimension \( \calN(\lambda) \) under the eigenvalue decay assumption (\aspref{asp:eigenvalue}). This result plays a crucial role in deriving the parameter constraints.
\begin{lemma}
    \label{lem:effective dimension}
    Under the eigenvalue decay condition \( t_{j} \asymp j^{-\beta} \) from \aspref{asp:eigenvalue}, we have
    \[
        \calN(\lambda)
        = \Tr{ (\LK + \lambda I)^{-1} \, \LK }
        \asymp \lambda^{-1/\beta}.
    \]
\end{lemma}
\begin{proof}
    Using the monotonicity of the function \( t \mapsto \frac{t}{t + \lambda} \), we obtain the bounds:
    \[
        \sum_{j \in N} \frac{c j^{-\beta}}{c j^{-\beta} + \lambda}
        \le \calN(\lambda)
        = \sum_{j \in N} \frac{t_{j}}{t_{j} + \lambda}
        \le \sum_{j \in N} \frac{C j^{-\beta}}{C j^{-\beta} + \lambda}.
    \]
    Approximating the sums by integrals yields:
    \[
        \int_{0}^{\infty} \frac{c (x + 1)^{-\beta}}{c (x + 1)^{-\beta} + \lambda} \, \dd x
        \le \calN(\lambda)
        \le \int_{0}^{\infty} \frac{C x^{-\beta}}{C x^{-\beta} + \lambda} \, \dd x.
    \]
    Applying the substitution \( v = \lambda^{1/\beta} x \), we obtain
    \begin{alignat*}{2}
        \int_{0}^{\infty} \frac{c (x + 1)^{-\beta}}{c (x + 1)^{-\beta} + \lambda} \, \dd x
         & = \lambda^{-1/\beta} \int_{\lambda^{1/\beta}}^{\infty} \frac{c}{c + v^{\beta}} \, \dd v
         &                                                                                         & \ge c_{\calN} \lambda^{-1/\beta}, \\
        \int_{0}^{\infty} \frac{C x^{-\beta}}{C x^{-\beta} + \lambda} \, \dd x
         & = \lambda^{-1/\beta} \int_{0}^{\infty} \frac{C}{C + v^{\beta}} \, \dd v
         &                                                                                         & \le C_{\calN} \lambda^{-1/\beta},
    \end{alignat*}
    where \( c_{\calN} \) and \( C_{\calN} \) are positive constants. Combining these inequalities gives \( \calN(\lambda) \asymp \lambda^{-1/\beta} \).
\end{proof}

The Bernstein inequality is employed repeatedly to control deviations between empirical means and their expectations:
\begin{lemma}[\cite{Caponnetto2007OptimalRR}, Proposition~2]
    \label{lem:bernstein}
    Let \( (\Omega, \calB, \rho) \) be a probability space and \( \xi = \xi(\omega) \) a random variable taking values in a separable Hilbert space \( \mathscr{H} \). Suppose that there exist positive constants \( \tilde{L} \) and \( \tilde{\sigma} \) such that
    \[
        \E{ \, \norm{ \xi - \E{ \xi } }_{\mathscr{H}}^{m} \, } \le \frac{1}{2} m! \tilde{L}^{m-2} \tilde{\sigma}^{2},
        \quad \forall m \ge 2.
    \]
    Then, for any i.i.d. sample \( \family{ \xi_{i} }_{i=1}^{n} \) and any \( \delta \in (0, 1) \), with probability at least \( 1 - \delta \),
    \[
        \norm{ \frac{1}{n} \sum_{i=1}^{n} \xi_{i} - \E{ \xi } }_{\mathscr{H}}
        \le 2 \LE( \frac{\tilde{L}}{n} + \frac{\tilde{\sigma}}{\sqrt{n}} \RI) \log \frac{2}{\delta}.
    \]
\end{lemma}
\begin{remark}
    Let \( \xi' \) be an independent copy of \( \xi \). Using Jensen's inequality, we obtain:
    \begin{align*}
        \E{ \, \norm{ \xi - \E{ \xi } }_{\mathscr{H}}^{m} \, }
         & \le \sE{ \, \sE{ \norm{ \xi - \xi' }_{\mathscr{H}}^{m} }{\xi'} \, }{\xi}
        \le 2^{m-1} \sE{ \, \sE{ \norm{ \xi }_{\mathscr{H}}^{m} + \norm{ \xi' }_{\mathscr{H}}^{m} }{\xi'} \, }{\xi} \\
         & = 2^{m} \E{ \norm{ \xi }_{\mathscr{H}}^{m} }.
    \end{align*}
    Consequently, if there exist positive constants \( \tilde{L} \) and \( \tilde{\sigma} \) satisfying
    \[
        \E{ \norm{ \xi }_{\mathscr{H}}^{m} }
        \le \frac{1}{2} m! \tilde{L}^{m-2} \tilde{\sigma}^{2},
        \quad \forall m \ge 2,
    \]
    then \( \E{ \, \norm{ \xi - \E{ \xi } }_{\mathscr{H}}^{m} \, } \le \frac{1}{2} m! (2 \tilde{L})^{m-2} (2 \tilde{\sigma})^{2} \). Applying \lemref{lem:bernstein} yields
    \[
        \norm{ \frac{1}{n} \sum_{i=1}^{n} \xi_{i} - \E{ \xi } }_{\mathscr{H}}
        \le 4 \LE( \frac{\tilde{L}}{n} + \frac{\tilde{\sigma}}{\sqrt{n}} \RI) \log \frac{2}{\delta}
    \]
    with probability at least \( 1 - \delta \).
\end{remark}

Next, \lemref{lem:cordes} (also known as the Cordes inequality) and \lemref{lem:A^s - B^s} establish bounds for operator powers:
\begin{lemma}[\cite{Cordes1987SpectralTL}, Lemma~5.1]
    \label{lem:cordes}
    Let \( A \) and \( B \) be positive bounded linear operators on a separable Hilbert space. For any \( h \in [0, 1] \), the following inequality holds:
    \[
        \norm{ A^{h} B^{h} }_{\textnormal{op}} \le \norm{ A B }_{\textnormal{op}}^{h},
    \]
    where \( \norm{ \cdot }_{\textnormal{op}} \) denotes the operator norm.
\end{lemma}

\begin{lemma}[\cite{Blanchard2010OptimalLR}, Lemma~E.3]
    \label{lem:A^s - B^s}
    Let \( A \) and \( B \) be positive self-adjoint operators such that \( \max \family{ \norm{ A }_{\textnormal{op}}, \norm{ B }_{\textnormal{op}} } \le U \). For any \( h > 0 \),
    \[
        \norm{ A^{h} - B^{h} } \le
        \begin{cases}
            \norm{ A - B }_{\textnormal{op}}^{h},       & h \le 1;
            \smallskip                                             \\
            h U^{h-1} \norm{ A - B }_{\textnormal{op}}, & h > 1.
        \end{cases}
    \]
\end{lemma}

Recall that in \aspref{asp:embedding}, the embedding index \( \alpha_{0} \) characterizes the embedding property of \( \ranH{\alpha} \) into \( L^{\infty}(\calX, \rhoTX) \). The following lemma provides an explicit expression for computing the embedding norm \( \norm{ \ranH{\alpha} \hookrightarrow L^{\infty}(\calX, \rhoTX) } \):
\begin{lemma}[\cite{Fischer2020SobolevNL}, Theorem~9]
    \label{lem:embedding norm}
    Assume \( \calH \) has embedding index \( \alpha_{0} < 1 \). For any \( \alpha > \alpha_{0} \), let \( M_{\alpha} = \norm{ \ranH{\alpha} \hookrightarrow L^{\infty}(\calX, \rhoTX) } \). Then,
    \[
        M_{\alpha}^{2} = \operatorname*{ess\,sup}_{x \in \calX} \sum_{j \in N} t_{j}^{\alpha} \, e_{j}^{2}(x),
    \]
    where \( \family{ t_{j}^{1/2} \, e_{j} }_{j \in N} \) forms an orthonormal basis for \( \calH = \ranH{1} \).
\end{lemma}

The following auxiliary result is also essential:
\begin{lemma}
    \label{lem:sup fraction}
    For any \( \lambda > 0 \) and \( h \in [0, 1] \),
    \[
        \sup_{t \ge 0} \frac{t^{h}}{t + \lambda} \le \lambda^{h-1}.
    \]
\end{lemma}
\begin{proof}
    The inequality is immediate for \( h = 0 \) or \( h = 1 \). For \( h \in (0, 1) \), consider the function
    \[
        t \mapsto \frac{t^{h}}{t + \lambda}.
    \]
    The derivative vanishes at \( t^{\ast} = \frac{h \lambda}{1 - h} \), which yields the maximum value
    \[
        \frac{(t^{\ast})^{h}}{t^{\ast} + \lambda} = \lambda^{h-1} \cdot h^{h} (1 - h)^{1 - h} \le \lambda^{h-1},
    \]
    since \( h^{h} (1 - h)^{1 - h} \le 1 \) for all \( h \in (0, 1) \).
\end{proof}

To prove \proref{pro:LKlam^1/2 hLKlam^-1/2}, we require two additional norm bounds from \lemref{lem:LKlam^-1/2 Kx} and \lemref{lem:LKlam^-1/2 (LK - hLK) LKlam^-1/2}:
\begin{lemma}
    \label{lem:LKlam^-1/2 Kx}
    Suppose that \( \calH \) has embedding index \( \alpha_{0} < 1 \). Then for any \( \alpha \in (\alpha_{0}, 1] \),
    \[
        \norm{ \LKlam^{-1/2} \, K(\cdot, x) }_{\calH}^{2}
        \le M_{\alpha}^{2} \lambda^{-\alpha},
        \quad \rhoTX \text{-a.e. } x \in \calX.
    \]
\end{lemma}
\begin{proof}
    Let \( \family{ t_{j}^{1/2} \, e_{j} }_{j \in N} \) be an orthonormal basis of \( \calH \). Expanding the squared norm yields:
    \begin{align*}
        \norm{ \LKlam^{-1/2} \, K(\cdot, x) }_{\calH}^{2}
         & = \norm{ \sum_{j \in N} \LE( \frac{t_{j}}{t_{j} + \lambda} \RI)^{1/2} e_{j}(x) \, t_{j}^{1/2} \, e_{j} }_{\calH}^{2}
        = \sum_{j \in N} \frac{t_{j}}{t_{j} + \lambda} e_{j}^{2}(x)                                                                \\
         & \le \sum_{j \in N} t_{j}^{\alpha} e_{j}^{2}(x) \cdot \LE( \sup_{j \in N} \frac{t_{j}^{1-\alpha}}{t_{j} + \lambda} \RI).
    \end{align*}
    The result follows by applying \lemref{lem:embedding norm} to bound the sum and \lemref{lem:sup fraction} to control the supremum term.
\end{proof}

\begin{lemma}
    \label{lem:LKlam^-1/2 (LK - hLK) LKlam^-1/2}
    Suppose that \aspref{asp:ratio} holds with \( p \in [1, \infty] \), and \( \calH \) has embedding index \( \alpha_{0} < 1 \). Then for any \( \alpha \in (\alpha_{0}, 1] \) and \( \delta \in (0, 1) \), with probability at least \( 1 - \delta \),
    \[
        \norm{ \LKlam^{-1/2} \, (\LK - \hLK) \, \LKlam^{-1/2} }
        \le 4 \LE( \frac{\tilde{L}_{1}}{n} + \frac{\tilde{\sigma}_{1}}{\sqrt{n}} \RI) \log \frac{2}{\delta},
    \]
    where
    \begin{align*}
        \tilde{L}_{1}
         & = L M_{\alpha}^{2} \cdot \lambda^{-\alpha},                                                                         \\
        \tilde{\sigma}_{1}
         & = \sigma M_{\alpha}^{1 + \frac{1}{p}} \cdot \lambda^{-\frac{1 + 1/p}{2} \alpha} \calN^{\frac{1 - 1/p}{2}}(\lambda).
    \end{align*}
\end{lemma}
\begin{proof}
    Define \( \xi = \xi(x) = \LKlam^{-1/2} \circ (w(x) \, \Kx \Kxa) \circ \LKlam^{-1/2} \), where \( \Kx \) and \( \Kxa \) are defined in \eqref{eq:Kx}. Let \( q \) be the conjugate exponent of \( p \), satisfying \( 1/p + 1/q = 1 \). Applying H\"{o}lder's inequality and \aspref{asp:ratio}, we obtain:
    \begin{align*}
        \E{ \norm{ \xi }_{\HS}^{m} }
         & = \int_{\calX} \norm{ \LKlam^{-1/2} \circ (\Kx \Kxa) \circ \LKlam^{-1/2} }_{\HS}^{m} w^{m-1}(x) \, \dd \rhoTx        \\
         & \le \LE( \int_{\calX} w^{p(m-1)}(x) \, \dd \rhoTx \RI)^{1/p}
        \cdot \LE( \int_{\calX} \norm{ \LKlam^{-1/2} \circ (\Kx \Kxa) \circ \LKlam^{-1/2} }_{\HS}^{qm} \, \dd \rhoTx \RI)^{1/q} \\
         & \le \frac{1}{2} m! L^{m-2} \sigma^{2}
        \cdot \LE( \int_{\calX} \norm{ \LKlam^{-1/2} \circ (\Kx \Kxa) \circ \LKlam^{-1/2} }_{\HS}^{qm} \, \dd \rhoTx \RI)^{1/q},
    \end{align*}
    where \( \norm{ \cdot }_{\HS} \) represents the Hilbert-Schmidt norm. Let \( \family{ t_{j}^{1/2} \, e_{j} }_{j \in N} \) be an orthonormal basis of \( \calH \), then:
    \begin{align*}
        \norm{ \LKlam^{-1/2} \circ (\Kx \Kxa) \circ \LKlam^{-1/2} }_{\HS}^{2}
         & = \sum_{j \in N} \norm{ \LKlam^{-1/2} \circ (\Kx \Kxa) \circ \LKlam^{-1/2} \, (t_{j}^{1/2} \, e_{j}) }_{\calH}^{2} \\
         & = \sum_{j \in N} \norm{ \LKlam^{-1/2} \LE(
            \inner{ K(\cdot, x), \LKlam^{-1/2} \, (t_{j}^{1/2} \, e_{j})}_{\calH}
            \, K(\cdot, x)
        \RI) }_{\calH}^{2}                                                                                                 \\
         & = \norm{ \LKlam^{-1/2} \, K(\cdot, x) }_{\calH}^{2}
        \cdot \sum_{j \in N} \inner{ K(\cdot, x), \LKlam^{-1/2} \, (t_{j}^{1/2} \, e_{j})}_{\calH}^{2}.
    \end{align*}
    Using the self-adjointness of \( \LKlam \), we have:
    \[
        \sum_{j \in N} \inner{ K(\cdot, x), \LKlam^{-1/2} \, (t_{j}^{1/2} \, e_{j})}_{\calH}^{2}
        = \sum_{j \in N} \inner{ \LKlam^{-1/2} \, K(\cdot, x), t_{j}^{1/2} \, e_{j})}_{\calH}^{2}
        = \norm{ \LKlam^{-1/2} \, K(\cdot, x) }_{\calH}^{2}.
    \]
    Hence,
    \[
        \norm{ \LKlam^{-1/2} \circ (\Kx \Kxa) \circ \LKlam^{-1/2} }_{\HS} = \norm{ \LKlam^{-1/2} \, K(\cdot, x) }_{\calH}^{2}.
    \]
    By \lemref{lem:LKlam^-1/2 Kx}, this quantity is uniformly bounded. Moreover, its integral satisfies:
    \begin{equation}
        \label{eq:int N(lambda)}
        \begin{aligned}
             & \eqspace \int_{\calX} \norm{ \LKlam^{-1/2} \, K(\cdot, x) }_{\calH}^{2} \, \dd \rhoTx                     \\
             & = \int_{\calX} \inner{ \LKlam^{-1/2} \, K(\cdot, x), \LKlam^{-1/2} \, K(\cdot, x) }_{\calH} \, \dd \rhoTx
            = \int_{\calX} \inner{ \LKlam^{-1} \, K(\cdot, x), K(\cdot, x) }_{\calH} \, \dd \rhoTx                       \\
             & = \int_{\calX} \Kxa \, \LKlam^{-1} \, \Kx \, \dd \rhoTx
            = \int_{\calX} \Tr{ \LKlam^{-1} \, (\Kx \Kxa) } \, \dd \rhoTx                                                \\
             & = \Tr{ \LKlam^{-1} \, \LK }
            = \calN(\lambda).
        \end{aligned}
    \end{equation}
    Consequently,
    \begin{align*}
        \LE( \int_{\calX} \norm{ \LKlam^{-1/2} \circ (\Kx \Kxa) \circ \LKlam^{-1/2} }_{\HS}^{qm} \, \dd \rhoTx \RI)^{1/q}
         & = \LE( \int_{\calX} \norm{ \LKlam^{-1/2} \, K(\cdot, x) }_{\calH}^{2qm} \, \dd \rhoTx \RI)^{1/q}             \\
         & \le \LE(
        (M_{\alpha}^{2} \lambda^{-\alpha})^{qm-1}
        \int_{\calX} \norm{ \LKlam^{-1/2} \, K(\cdot, x) }_{\calH}^{2} \, \dd \rhoTx
        \RI)^{1/q}                                                                                                      \\
         & = M_{\alpha}^{2 \LE( m - \frac{1}{q} \RI)} \lambda^{-\alpha \LE( m - \frac{1}{q} \RI)} \calN^{1/q}(\lambda).
    \end{align*}
    Combining these bounds yields:
    \begin{align*}
        \E{ \norm{ \xi }_{\HS}^{m} }
         & \le \frac{1}{2} m! L^{m-2} \sigma^{2}
        \cdot M_{\alpha}^{2 \LE( m - \frac{1}{q} \RI)} \lambda^{-\alpha \LE( m - \frac{1}{q} \RI)} \calN^{1/q}(\lambda) \\
         & = \frac{1}{2} m!
        \cdot { \underbrace{\LE( L M_{\alpha}^{2} \lambda^{-\alpha} \RI)}_{\tilde{L}_{1}} }^{m-2}
        \cdot { \underbrace{\LE( \sigma M_{\alpha}^{2-\frac{1}{q}} \lambda^{-\alpha \LE( 1 - \frac{1}{2q} \RI)} \calN^{\frac{1}{2q}}(\lambda) \RI)}_{\tilde{\sigma}_{1}} }^{2}.
    \end{align*}
    Applying \lemref{lem:bernstein} with parameters \( \tilde{L}_{1} \) and \( \tilde{\sigma}_{1} \), and noting that \( \norm{ \cdot } \le \norm{ \cdot }_{\HS} \), we conclude the proof.
\end{proof}

The following bound for \( \norm{ \hLK - \LK } \) supports the proof of \proref{pro:hLKlam^1/2 psilam(hLKlam) flam}:
\begin{lemma}
    \label{lem:hLK - LK}
    Under \aspref{asp:ratio} with \( p \in [1, \infty] \), for any \( \delta \in (0, 1) \), with probability at least \( 1 - \delta \),
    \[
        \norm{ \hLK - \LK }
        \le 4 \kappa^{2} \LE( \frac{L}{n} + \frac{\sigma}{\sqrt{n}} \RI) \log \frac{2}{\delta}.
    \]
\end{lemma}
\begin{proof}
    Define \( \xi(x) = w(x) \Kx \Kxa \), so that \( \hLK = \frac{1}{n} \sum_{i=1}^{n} \xi_{i} \) with \( \xi_{i} = \xi(x_{i}) \).  We estimate the moment bound:
    \[
        \E{ \norm{ \xi }_{\HS}^{m} }
        = \int_{\calX} \norm{ \Kx \Kxa }_{\HS}^{m} w^{m-1}(x) \, \dd \rhoTx
        \le \kappa^{2m} \int_{\calX} w^{m-1}(x) \, \dd \rhoTx.
    \]
    Applying H\"{o}lder's inequality and \aspref{asp:ratio} gives:
    \[
        \int_{\calX} w^{m-1}(x) \, \dd \rhoTx
        \le 1 \cdot \LE( \int_{\calX} w^{p(m-1)}(x) \, \dd \rhoTx \RI)^{1/p} \le \frac{1}{2} m! L^{m-2} \sigma^{2}.
    \]
    By \lemref{lem:bernstein}, we obtain:
    \[
        \norm{ \hLK - \LK }_{\HS}
        \le 4 \kappa^{2} \LE( \frac{L}{n} + \frac{\sigma}{\sqrt{n}} \RI) \log \frac{2}{\delta}
    \]
    with probability at least \( 1 - \delta \). The result follows since \( \norm{ \hLK - \LK } \le \norm{ \hLK - \LK }_{\HS} \).
\end{proof}

Recall that in \thmref{thm:truncated weighting}, we introduced a novel estimator based on the truncated density ratio \( \wD \). The following lemma quantifies the approximation error between \( \wD \) and the true density ratio \( w \):
\begin{lemma}
    \label{lem:(T) w - wd}
    Suppose the density ratio \( w \) satisfies \aspref{asp:ratio} with \( p \in [1, \infty) \). Then the following inequality holds:
    \[
        \int_{\calX} ( \, w(x) - \wD(x) \, )^{2} \, \dd \rhoSx
        \le \LE( D^{-(m-1)} \frac{1}{2} m! L^{m-2} \sigma^{2} \RI)^{p},
        \quad \forall m \ge 2.
    \]
\end{lemma}
\begin{proof}
    By definition, \( \wD(x) = \min \{ w(x), D \} \). Direct computation yields:
    \begin{align*}
        \int_{\calX} ( \, w(x) - \wD(x) \, )^{2} \, \dd \rhoSx
         & = \int_{\calX} \LE( 1 - \frac{\wD(x)}{w(x)} \RI)^{2} \, \dd \rhoTx
        = \int_{\family{x: w(x) \ge D}} \LE( 1 - \frac{D}{w(x)} \RI)^{2} \, \dd \rhoTx \\
         & \le \int_{\family{x: w(x) \ge D}} 1 \, \dd \rhoTx
        = \rhoTX \LE( \family{x: w(x) \ge D} \RI).
    \end{align*}
    Applying Markov's inequality and invoking \aspref{asp:ratio} gives:
    \begin{align*}
        \rhoTX \LE( \family{x: w(x) \ge D} \RI)
         & \le D^{-p(m-1)} \int_{\calX} w^{p(m-1)}(x) \, \dd \rhoTx
        \le \LE( D^{-(m-1)} \frac{1}{2} m! L^{m-2} \sigma^{2} \RI)^{p}.
    \end{align*}
    This completes the proof.
\end{proof}

To establish \proref{pro:(T) LKlamD^1/2 hLKlamD^-1/2}, we apply \lemref{lem:(T) w - wd} to derive the following operator norm bound:
\begin{lemma}
    \label{lem:(T) (LK - LKD) LKlam^-1}
    Assume \aspref{asp:ratio} holds with \( p \in [1, \infty) \). Then:
    \[
        \norm{ (\LK - \LKD) \, \LKlam^{-1} }
        \le \kappa \lambda^{-1/2} \cdot \calN^{1/2}(\lambda) \cdot \LE( D^{-(m-1)} \frac{1}{2} m! L^{m-2} \sigma^{2} \RI)^{p/2},
        \quad \forall m \ge 2.
    \]
\end{lemma}
\begin{proof}
    Beginning with the operator norm definition:
    \begin{align*}
        \norm{ (\LK - \LKD) \, \LKlam^{-1} }
         & = \sup_{\norm{ f }_{\calH} \le 1} \norm{ (\LK - \LKD) \, \LKlam^{-1} \, f }_{\calH}                                                   \\
         & = \sup_{\norm{ f }_{\calH} \le 1} \norm{ \int_{\calX} ( \, w(x) - \wD(x) \, ) \, \Kx \Kxa \, \LKlam^{-1} \, f \, \dd \rhoSx }_{\calH} \\
         & \le \sup_{\norm{ f }_{\calH} \le 1} \int_{\calX} ( \, w(x) - \wD(x) \, ) \norm{ \Kx \Kxa \, \LKlam^{-1} \, f }_{\calH} \, \dd \rhoSx.
    \end{align*}
    For any \( f \) with \( \norm{ f }_{\calH} \le 1 \), we bound:
    \begin{align*}
        \norm{ \Kx \Kxa \, \LKlam^{-1} \, f }_{\calH}
         & = \norm{ \inner{ \LKlam^{-1} \, f, K(\cdot, x) }_{\calH} K(\cdot, x) }_{\calH}
        = \norm{ \inner{ f, \LKlam^{-1} \, K(\cdot, x) }_{\calH} K(\cdot, x) }_{\calH}                                 \\
         & \le \norm{ f }_{\calH} \cdot \norm{ \LKlam^{-1} \, K(\cdot, x) }_{\calH} \cdot \norm{ K(\cdot, x) }_{\calH}
        \le \kappa \norm{ \LKlam^{-1} \, K(\cdot, x) }_{\calH}.
    \end{align*}
    Furthermore,
    \[
        \norm{ \LKlam^{-1} \, K(\cdot, x) }_{\calH}
        \le \norm{ \LKlam^{-1/2} } \cdot \norm{ \LKlam^{-1/2} \, K(\cdot, x) }_{\calH}
        \le \lambda^{-1/2} \norm{ \LKlam^{-1/2} \, K(\cdot, x) }_{\calH}.
    \]
    Returning to the main bound and applying the Cauchy-Schwarz inequality:
    \begin{align*}
         & \eqspace \norm{ (\LK - \LKD) \, \LKlam^{-1} }                                                                                                                                                             \\
         & \le \kappa \lambda^{-1/2} \cdot \int_{\calX} ( \, w(x) - \wD(x) \, ) \, \norm{ \LKlam^{-1/2} \, K(\cdot, x) }_{\calH} \, \dd \rhoSx                                                                       \\
         & \le \kappa \lambda^{-1/2} \cdot \LE( \int_{\calX} \norm{ \LKlam^{-1/2} \, K(\cdot, x) }_{\calH}^{2} \, \dd \rhoSx \RI)^{1/2} \cdot \LE( \int_{\calX} ( \, w(x) - \wD(x) \, )^{2} \, \dd \rhoSx \RI)^{1/2} \\
         & \le \kappa \lambda^{-1/2} \cdot \calN^{1/2}(\lambda) \cdot \LE( D^{-(m-1)} \frac{1}{2} m! L^{m-2} \sigma^{2} \RI)^{p/2},
    \end{align*}
    where the last inequality uses \lemref{lem:(T) w - wd} and the identity \eqref{eq:int N(lambda)}:
    \[
        \int_{\calX} \norm{ \LKlam^{-1/2} \, K(\cdot, x) }_{\calH}^{2} \, \dd \rhoSx
        = \calN(\lambda).
    \]
    This completes the proof.
\end{proof}

Following the methodology of \proref{pro:hLKlam^1/2 psilam(hLKlam) flam}, we bound \( \norm{ \LK - \hLKD } \) to prove \proref{pro:(T) hLKlamD^1/2 psilam(hLKlamD) flam}:
\begin{lemma}
    \label{lem:(T) LK - hLKD}
    Suppose that the density ratio \( w \) satisfies \aspref{asp:ratio} with \( p \in [1, \infty) \). Then for any \( \delta \in (0, 1) \), with probability at least \( 1 - \delta \):
    \[
        \norm{ \LK - \hLKD }
        \le \kappa^{2} \LE( D^{-(m-1)} \frac{1}{2} m! L^{m-2} \sigma^{2} \RI)^{p} + 4 \kappa^{2} \LE( \frac{L}{n} + \frac{\sigma}{\sqrt{n}} \RI) \log \frac{6}{\delta},
        \quad \forall m \ge 2.
    \]
\end{lemma}
\begin{proof}
    Decompose the norm as \( \norm{ \LK - \hLKD } \le \norm{ \LK - \LKD } + \norm{ \LKD - \hLKD } \). For the first term:
    \[
        \begin{aligned}
            \norm{ \LK - \LKD }
             & = \norm{ \int_{\calX} ( \, w(x) - \wD(x) \, ) \, \Kx \Kxa \, \dd \rhoSx }                                                     \\
             & \le \int_{\calX} ( \, w(x) - \wD(x) \, ) \norm{ \Kx \Kxa } \, \dd \rhoSx                                                      \\
             & \le \LE( \int_{\calX} \norm{ \Kx \Kxa }^{2} \RI)^{1/2} \LE( \int_{\calX} ( \, w(x) - \wD(x) \, )^{2} \, \dd \rhoSx \RI)^{1/2} \\
             & \le \kappa^{2} \LE( D^{-(m-1)} \frac{1}{2} m! L^{m-2} \sigma^{2} \RI)^{p/2}.
        \end{aligned}
    \]
    The second inequality follows from Cauchy-Schwarz, and the final bound uses \lemref{lem:(T) w - wd} and the fact that
    \[
        \norm{ \Kx \Kxa } \le \Tr{ \Kx \Kxa } = \Tr{ \Kxa \Kx } = K(x, x) \le \kappa^{2}.
    \]
    For the second term, since \( \wD(x) \le w(x) \), we adapt the proof of \lemref{lem:hLK - LK} to obtain:
    \[
        \norm{ \LKD - \hLKD }
        \le 4 \kappa^{2} \LE( \frac{L}{n} + \frac{\sigma}{\sqrt{n}} \RI) \log \frac{6}{\delta}
    \]
    with probability \( 1 - \delta \). Combining both bounds yields the result.
\end{proof}

\end{document}